%% file: ms.tex
\documentclass[twoside]{article}

\usepackage{aistats2020}
\usepackage{enumitem}
\usepackage{wrapfig}
\usepackage{amsthm}

\usepackage{amsthm}
\usepackage{graphicx}
\usepackage{wrapfig}

\usepackage{amssymb}

\usepackage{color}
\usepackage{mathtools}

\usepackage{multirow}

\usepackage{algorithm}
\usepackage{algorithmic}

\usepackage[most]{tcolorbox}

\newcommand{\denselist}{\itemsep -2pt\partopsep 0pt}

\usepackage{thmtools} 
\usepackage{thm-restate}

\newcommand{\E}{\mathop{\mathbb{E}}} 
\newcommand{\base}{\mathrm{base}}

\newtheorem{theorem}{Theorem}[section]
\newtheorem{corollary}{Corollary}[theorem]

\usepackage{booktabs}

\begin{document}

\global\long\def\marker{\checkmark}
\global\long\def\m{\checkmark}

%

%

\twocolumn[

\aistatstitle{A Rule for Gradient Estimator Selection, with an Application to Variational Inference}

\aistatsauthor{ Tomas Geffner \And Justin Domke}

\aistatsaddress{ University of Massachusetts, Amherst \And University of Massachusetts, Amherst} ]

\begin{abstract}
\input{./sections/abstract.tex}
\end{abstract}

\input{./sections/introduction.tex}

\input{./sections/preliminaries.tex}

\input{./sections/g2tprinciple2.tex}

\input{./sections/miqcp2.tex}

\input{./sections/experiments2.tex}
\input{./sections/conclusions.tex}

\newpage
\clearpage
\bibliography{control}
\bibliographystyle{plain}
\clearpage

\input{./sections/appendix.tex}








\end{document}

%% file: sections/abstract.tex
Stochastic gradient descent (SGD) is the workhorse of modern machine learning. Sometimes, there are many different potential gradient estimators that can be used. When so, choosing the one with the best tradeoff between cost and variance is important. This paper analyzes the convergence rates of SGD as a function of time, rather than iterations. This results in a simple rule to select the estimator that leads to the best optimization convergence guarantee. This choice is the same for different variants of SGD, and with different assumptions about the objective (e.g. convexity or smoothness). Inspired by this principle, we propose a technique to automatically select an estimator when a finite pool of estimators is given. Then, we extend to infinite pools of estimators, where each one is indexed by control variate weights. This is enabled by a reduction to a mixed-integer quadratic program. Empirically, automatically choosing an estimator performs comparably to the best estimator chosen with hindsight.

%% file: sections/introduction.tex

\section{Introduction}

Convergence guarantees for stochastic optimization algorithms depend on structural properties of the objective, such as convexity and smoothness, the variance of the gradient estimates, and the number of iterations performed. Naturally, lower variance and more iterations result in better guarantees \cite{shamir_SGDopt, shamir_lastiter_avgs, bottou_SGDreview}. Given a fixed total optimization time, fewer iterations are possible with a slower estimator. Therefore, an estimator with low variance and low computational cost is ideal.


This paper is inspired by stochastic gradient variational inference (SGVI), where there are multiple gradient estimators with varying costs and variances. Estimators may be obtained using the reparameterization trick \cite{vaes_welling, rezende2014stochastic, doublystochastic_titsias}, the score function method \cite{williams_reinforce, blackbox_blei}, or other techniques \cite{localexpectations_gredilla, stickingthelanding, schulman2015gradient, generalreparam_blei, agakov2004auxiliary}. Also, many control variates can be added to the estimator to reduce variance \cite{traylorreducevariance_adam, backpropvoid, neuralVI_minh, viasstochastic_jordan, REBAR, CVs}.\footnote{In fact, it has been shown that the problem of selecting a gradient estimator for variational inference can be phrased as choosing a set of control variates to use along with their weights \cite{CVs}.}

The cost and variance of an estimator significantly affects optimization convergence speed. Different estimators lead to different performances, and the optimal estimator is often situation-dependent. As an example, Fig. \ref{fig:intro} shows the results of running SGVI on several models (described in Sec. \ref{sec:exp_details}). We compare three common gradient estimators:
\begin{itemize} \denselist
\item (Rep) Plain reparameterization estimator,
\item (Miller) Estimator proposed by Miller et al. \cite{traylorreducevariance_adam}, which adds a Taylor expansion control variate,
\item (STL) ``Sticking-the-landing'' estimator \cite{stickingthelanding}, which avoids computing certain terms.
\end{itemize}
(Detailed descriptions of the estimators are in Section \ref{sec:g2tpractice}.) Table \ref{table:intro} gives the time costs of each estimator. There is no universal best estimator. Different estimators perform better in different situations. An estimator's variance depends on the model, dataset, and current parameters, while the time cost also depends on the implementation and computational platform. Rather than rely on the user to navigate these tradeoffs, we propose that estimator selection could be done adaptively. This paper investigates how, given a pool of gradient estimators, \emph{automatically} choose one to get the best convergence guarantee.

Section \ref{sec:back} introduces some background on SGVI and control variates. Section \ref{sec:G2T} studies cost-variance tradeoffs by analyzing the convergence rates of several variants of SGD. We do this with or without the structural properties of \emph{convexity}, \emph{smoothness}, or \emph{strong convexity}. We observe what we call the ``$G^2 T$ principle''. This states that, given a set of candidate estimators, the best convergence guarantee is obtained by the estimator for which $G^2 T$ is minimum, where $G^2 \geq \E \Vert g \Vert^2$ denotes a bound on an estimator's expected squared norm, and $T$ its time cost. Put another way, estimators with lower $G^2 T$ lead to better convergence guarantees. This is true regardless of the objective's structural properties (Section \ref{sec:g2t}). 

In theory, this principle could be directly used if a pool of estimators with known values for $G$ and $T$ was given. In practice, these values are rarely known a priori. In Section \ref{sec:estselec} we propose an estimator selection algorithm that uses estimates of these quantities instead of the true values.

Finally, motivated by SGVI where control variates play an important role, we look at the problem of control variate selection. If used with the right weights, control variates lead to a reduction in the estimator’s variance. However, they carry an extra computational cost, which may or may not be worth paying depending on the situation. In Section \ref{sec:miqcp} we consider the case where there is an infinite pool of estimators, each indexed by a set of control variate weights. Selecting an estimator involves selecting what control variates to use along with their weights. We show that the problem of minimizing the estimated $G^2 T$ values over this infinite set of estimators can be reduced to a small mixed integer quadratically constrained program (MIQCP) and then solved with off the shelf methods. Section \ref{sec:exps} gives an experimental evaluation.


\begin{table}[]
\centering
\begin{tabular}{llll}
\toprule
\multirow{2}{*}{\textbf{Model}} & \multicolumn{3}{c}{\textbf{Estimator}} \\ \cmidrule{2-4}
                       & Rep           & Miller          & STL          \\ \midrule
BNN-A                  & 8.2           & 25              & 11           \\
BNN-B                  & 7.1           & 21              & 9.5          \\
Log Reg (a1a)          & 6.9           & 120             & 12           \\
Hier Poisson           & 2.6           & 7.9             & 3.8          \\
\bottomrule
\end{tabular}
\caption{\textbf{Different estimators have different costs.} Time costs (in ms) for estimators in Fig. \ref{fig:intro}.}
\label{table:intro}
\end{table}

\begin{figure}[t]
  \centering
  \includegraphics[width=0.49\linewidth]{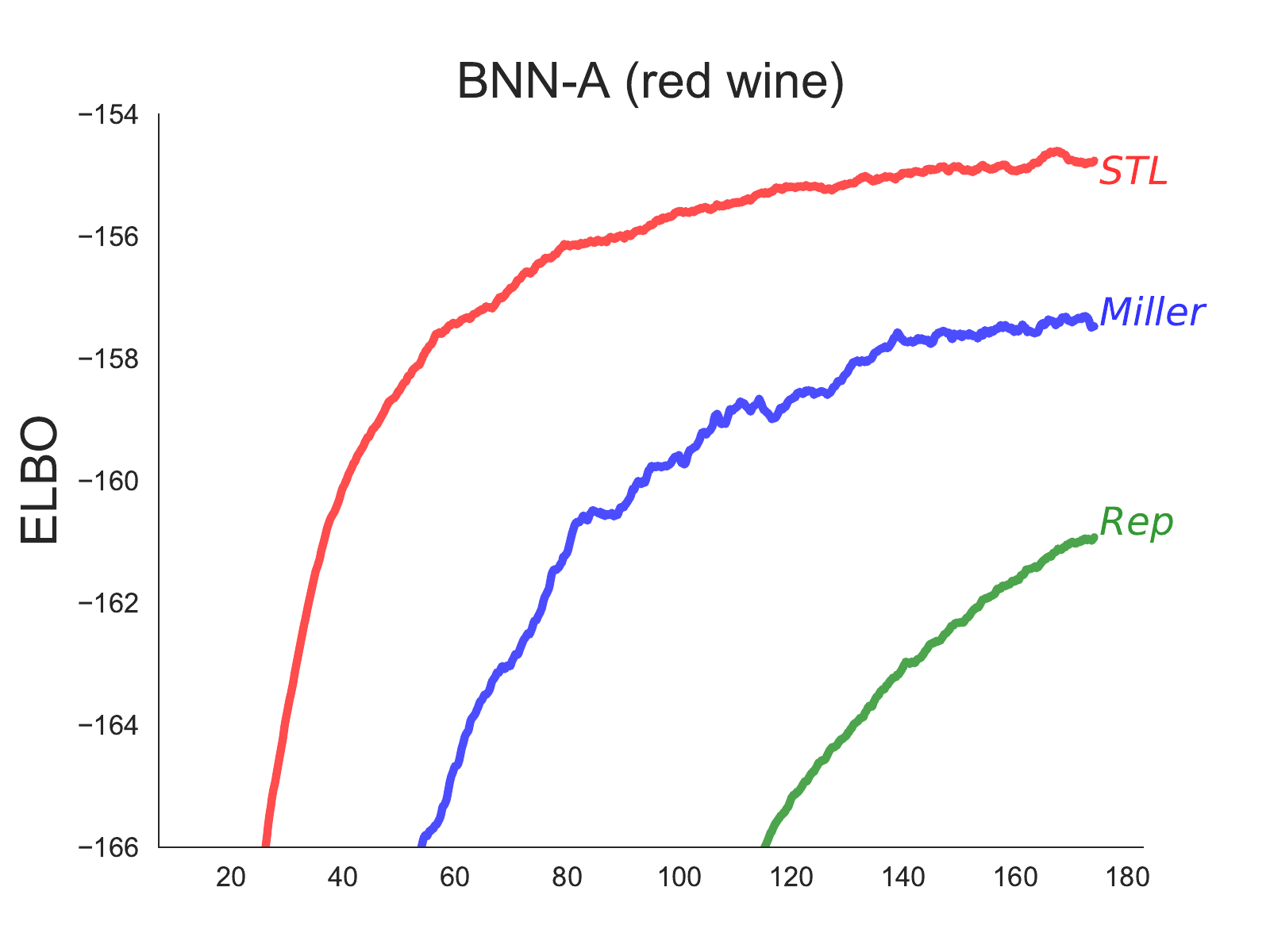}
  \includegraphics[width=0.49\linewidth]{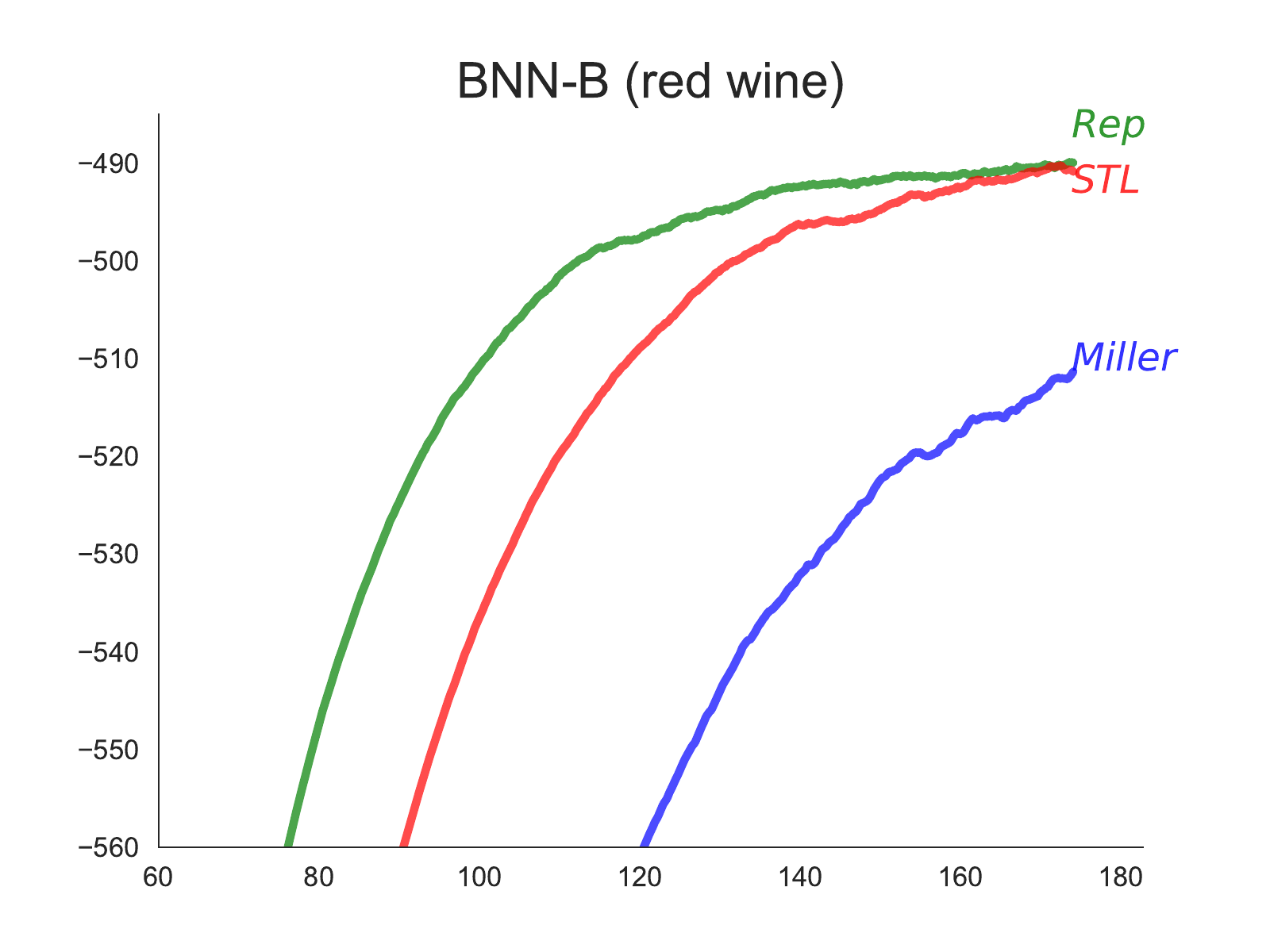}

  \includegraphics[width=0.49\linewidth]{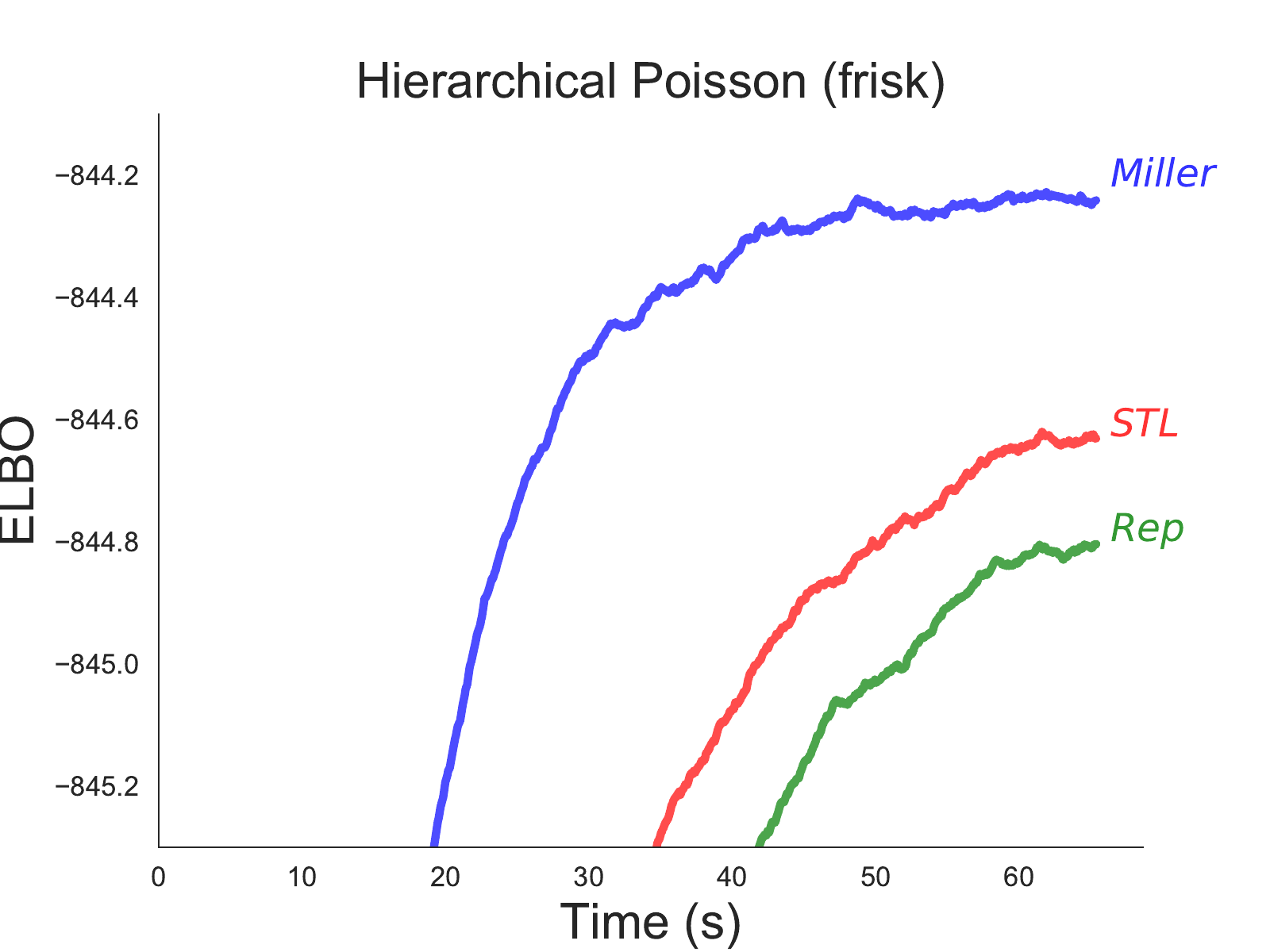}
  \includegraphics[width=0.49\linewidth]{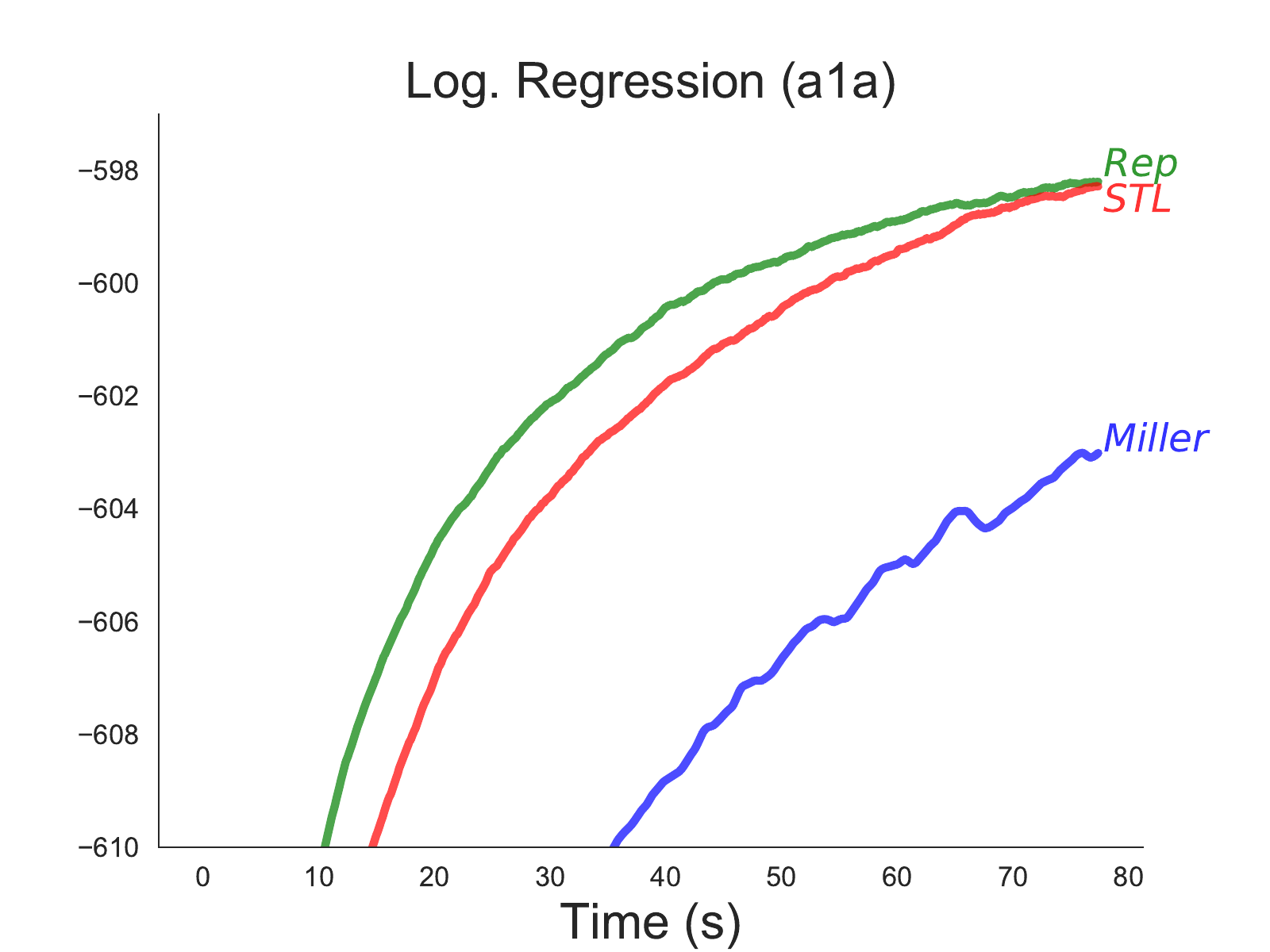}
  \caption{\textbf{Different estimators are better in different situations.} SGVI on four models. Each plot compares three different gradient estimators, described in the text. More details on the optimization algorithm and models are in Section \ref{sec:exps}.}
  \label{fig:intro}
\end{figure}

\subsection{Contributions}

\begin{description}
\item[$\mathbf{G^2 T}$ principle:] Given a set of gradient estimators with different $G^2$ and $T$ values, the best convergence guarantee is achieved by the estimator with minimum $G^2 T$. The same estimator is optimal \emph{regardless} of structural properties of the objective and with different variants of SGD (Section \ref{sec:g2t}).

\item[Gradient estimator selection:] Given a finite set of gradient estimators, we propose an algorithm to automatically select one using empirical estimates of $G^2$ and $T$. (Algorithm \ref{alg:g2tpractice}, Section \ref{sec:estselec}).

\item[Control variate selection via MIQCP:] Given an infinite set of estimators indexed by control variate weights, we propose an algorithm to automatically select a gradient estimator using empirical estimates of $G^2$ and $T$. This is based on a reduction to a small MIQCP, which can be solved efficiently in practice. (Algorithm \ref{alg:g2tcv}, Section \ref{sec:miqcp}).

\item[Empirical validation] We apply the automatic estimator selection techniques to Bayesian logistic regression, hierarchical regression, and Bayesian neural network models. Empirically, automatic estimator selection performs similarly to choosing the best estimator with hindsight with finite pools of estimators (Section \ref{sec:g2tpractice}) and similarly with infinite pools of estimators (Section \ref{sec:exps}).
\end{description}

\textbf{Notation:} We use $g(w, \xi)$, where $\xi$ is a random variable, to denote an unbiased estimator of target's gradient, $G^2(g)$ to denote a bound on $g$'s expected squared norm, $\E_{\xi}||g(w, \xi)||^2 \leq G^2(g) \,\, \forall w$, and $T(g)$ to denote the computational cost of computing estimator $g(w, \xi)$, measured in seconds. We drop the dependencies of $G^2$ and $T$ on $g$ when clear from context.

%% file: sections/preliminaries.tex

\section{Preliminaries}\label{sec:back}

\subsection{Stochastic Gradient Variational Inference}

Given a model $p(x, z)$, where $x$ is observed data and $z$ are latent variables, the goal of variational inference (VI) is to find parameters $w$ to approximate $p(z|x)$ with a simpler distribution $q_w(z)$. Since
\[ \log p(x) = \underbrace{\E_{Z\sim q_w(z)} \left[ \log\frac{p(x,Z)}{q_w(Z)} \right]}_{\mbox{ELBO}(w)} + \mbox{KL}(q_w(Z)||p(Z|x)), \]
minimizing the KL-divergence is equivalent to maximizing the $\mbox{ELBO}(w)$ (Evidence Lower BOund).

When $p$ and $q$ are complex, it is typically necessary to optimize the ELBO using stochastic methods \cite{neuralVI_minh, viasstochastic_jordan, blackbox_blei, doublystochastic_titsias, VIforMCobjectives_mnih}. These require an unbiased estimator of $\nabla_w \mbox{ELBO}(w)$, which can be expressed as
\begin{align}
 & \nabla_w \mbox{ELBO} = \nabla_w \E_{q_w(z)} \log p(x,Z) + \nabla_w \mathcal{H}(q_w) \label{eq:rep}\\
 & = \nabla_w \E_{q_w(z)} \log p(x,Z) \nonumber\\
 & - \left[\nabla_w \E_{q_w(z)} \log q_v(Z) + \nabla_w \E_{q_v(z)} \log q_w(Z) \right]_{v = w} \label{eq:stl}.
\end{align}
Estimates for this gradient can be obtained through the reparameterization trick \cite{vaes_welling, doublystochastic_titsias, rezende2014stochastic}, the generalized reparameterization trick \cite{generalreparam_blei}, and the score function method \cite{williams_reinforce}, among other techniques \cite{blackbox_blei, traylorreducevariance_adam, neuralVI_minh, VIforMCobjectives_mnih, viasstochastic_jordan, REBAR}. All of these express the gradient as an expectation and approximate it using Monte Carlo sampling. For example, the popular reparameterization trick is based on finding a function $\mathcal{T}_w$ that transforms $\xi \sim q_0$ into $Z = \mathcal{T}_w(\xi) \sim q_w$. Then, it approximates terms like $\nabla_w \E_{q_w(z)} \log p(x, Z)$ with the unbiased estimator $\nabla_w \log p(x,\mathcal{T}_w(\xi))$.

\subsection{Control Variates for Variational Inference} \label{sec:CVSVI}

Control variates \cite{mcbook} are random variables with mean zero that can be used to reduce the variance of an estimator. They are widely used in SGVI \cite{CVs, backpropvoid, traylorreducevariance_adam, neuralVI_minh, viasstochastic_jordan, blackbox_blei, REBAR}. Take an estimator, which is some function $g_{\base}$ such that $\E_\xi g_{\base}(w, \xi) = \nabla_w \mbox{ELBO}(w)$. Take also the control variate $c$ such that $\E_\xi c(w,\xi)=0$. The final gradient estimator is $g(w, \xi) = g_{\base}(w, \xi) + a \, c(w, \xi)$, where $a$ is a scalar. When multiple control variates $c_1,...,c_J$ are available, the final gradient estimator can be expressed as
\begin{equation} \textstyle g(w, \xi) = g_{\base}(w, \xi) + \sum_{i = 1}^J a_i \, c_i(w, \xi). \end{equation}
Equivalently, we could write $g = g_{\base} + C a$, where $C(w, \xi)$ is a matrix with control variate $c_i$ as its $i$-th column, and $a$ is a vector containing the weights for each control variate.

As a concrete example, Miller et al. \cite{traylorreducevariance_adam} proposed a control variate to reduce the variance of the reparameterization estimator $g_\base(w,\xi) = \nabla_w \log p(x,\mathcal{T}_w(\xi))$. Given a second-order Taylor approximation $u(z) \approx \log p(x,z)$, they suggest the control variate $c(w,\xi) = \nabla_w \E_{q_w} u(Z) - \nabla_w u(\mathcal{T}_w(\xi))$. Since $u$ is quadratic, this can be computed exactly when $q_w$ is Gaussian. If the Taylor expansion is accurate, $c$ will greatly reduce the estimator's variance.

Other control variates proposed can be built using upper/lower bounds of $\log p(x, z)$ instead of a Taylor expansion \cite{viasstochastic_jordan}, baselines \cite{neuralVI_minh, blackbox_blei}, or the difference between two unbiased estimates of a term (expectation of some function) \cite{CVs, backpropvoid, REBAR}, among others. 

%% file: sections/g2tprinciple2.tex

\section{The $\mathbf{G^2T}$ principle}
\label{sec:G2T}

\subsection{Stochastic Optimization Review}

Many different first order stochastic optimization algorithms can be used to minimize an objective function. The simplest one, SGD, requires an initial guess $w_0$, a step-size $\eta_k$, and an unbiased estimator of the objective's gradient $g(w_k, \xi_k)$. It uses the update rule
\begin{equation} w_{k+1} = w_k - \eta_{k}\  g(w_k,\xi_k). \end{equation}
Table \ref{table:SGD} shows guarantees for different variants of SGD. These usually require two things: a bound on $\E \Vert g(w,\xi) \Vert^2$, and some structural assumption about the objective (the most common being convexity, $\lambda$-strong convexity, and $L$-smoothness).

\renewcommand{\arraystretch}{1.6}
\begin{table*}[t]
\begin{small}
  \centering
  \begin{tabular}{cccll}
    \toprule
   
    \multicolumn{3}{c}{\textbf{Assumptions Needed}} & \multirow{2}{*}{\textbf{Algorithm}} & \multirow{2}{*}{\textbf{Guarantee}}\\ \cmidrule{1-3}
    Convex & $\lambda$-strongly convex & $L$-smooth & & \\
    \midrule
    
    $\marker$ & $\marker$ & $\marker$ & SGD, $\eta_k = 1 / (\lambda k)$ \cite{shamir_SGDopt} & $\E F(w_K) - F(w^*) \leq \frac{2 L}{\lambda^2} \frac{G^2}{K}$ \\ \midrule

    $\marker$ & $\marker$ & $-$ & SGD, $\eta_k = 1 / (\lambda k)$ \cite{shamir_SGDopt} & $\E ||w_K - w^*||^2 \leq \frac{4}{\lambda^2} \frac{G^2}{K}$ \\ \hline

    $\marker$ & $-$ & $-$ & SGD, $\eta_k = \frac{D_w}{G\sqrt{K}}$ \cite{nemirovski2009robust} & $\E F(\bar{w}) - F(w^*) \leq D_w \frac{G}{\sqrt{K}}$\\ \midrule

    $-$ & $-$ & $\marker$ & SGD, $\eta_k = \sqrt{\frac{2 D_f}{L K G^2}}$ \cite{bottou_SGDreview} & $\E \frac{1}{K} \sum_{i = 1}^K ||\nabla F(w_i)||^2 \leq \sqrt{L D_f} \frac{G}{\sqrt{K}}$ \\ \midrule

    \multirow{2}{*}{$-$} & \multirow{2}{*}{$-$} & \multirow{2}{*}{$\marker$} & SGD+Momentum ($\beta$) \cite{yang_momentum}, &
    \multicolumn{1}{l}{$\E \frac{1}{K} \sum_{i = 1}^K ||\nabla F(w_i)||^2 \leq$} \\

    & & & $\eta_k = \sqrt{\frac{2 D_f (1 - \beta)^4}{\left(\beta^2 + (1-\beta)^2\right)K L G^2}}$ &
    \multicolumn{1}{r}{$\sqrt{\frac{8 D_f L \left(\beta^2 + (1 - \beta)^2\right)}{(1-\beta)^2}} \frac{G}{\sqrt{K}}$} \\ \midrule

    \multirow{2}{*}{$-$} & \multirow{2}{*}{$-$} & \multirow{2}{*}{$\marker$} & SGD+Nesterov ($\beta$) \cite{yang_momentum}, &
    \multicolumn{1}{l}{$\E \frac{1}{K} \sum_{i = 1}^K ||\nabla F(w_i)||^2 \leq$} \\

    & & & $\eta_k = \sqrt{\frac{2 D_f (1 - \beta)^4}{\left(\beta^4 + (1-\beta)^2\right)K L G^2}}$ &
    \multicolumn{1}{r}{$\sqrt{\frac{8 D_f L \left(\beta^4 + (1 - \beta)^2\right)}{(1-\beta)^2}} \frac{G}{\sqrt{K}}$} \\ \bottomrule
  
  \end{tabular}
  \vspace{0.1cm}
  \caption{Convergence rates for stochastic optimization algorithms under different assumptions about the objective. Constants in the results are \small $D_f = F(w_0) - F(w^*)$ \normalsize and \small $D_w = ||w_0 - w^*||$ \normalsize. $K$ is the number of optimization steps, and $G^2$ is such that \small $\E_\xi||g(w, \xi)||^2 \leq G^2 \,\, \forall w$ \normalsize. The learning rates shown are the optimal ones.}
  \label{table:SGD}
\end{small}  
\end{table*}

It will be useful for the analysis in the following section to represent each of the guarantees in Table \ref{table:SGD} as a function $\Theta(\lambda, L, C, K, G^2)$, where $\lambda$ is the strong-convexity parameter (or 0), $L$ is the smoothness parameter (or $\infty$), $C$ is a boolean representing whether the problem is convex, $K$ is the number of iterations, and $G^2$ represents the bound on the estimator's expected squared norm. 
For example, if the target is $\lambda$-strongly convex and $L$-smooth, then $\Theta(\lambda, L, C, K, G^2) = \frac{2 L}{\lambda^2} \frac{G^2}{K}$ (first row of Table \ref{table:SGD}). All bounds presented in Table \ref{table:SGD} can be written as
\begin{equation} \textstyle \Theta(\lambda, L, C, K, G^2) = \alpha(\lambda, L, C) \left( \frac{G^2}{K} \right)^p, \label{eq:GK} \end{equation}
where $p>0$ and $\alpha$ depends on the properties of the target and initialization, but \textit{not on the gradient estimator used}.

\subsection{The $\mathbf{G^2T}$ principle} \label{sec:g2t}


Given a set of gradient estimators with varying costs and variances, our goal is to find the one that gives the best convergence guarantee for optimization algorithms. Given a total optimization time $T_\mathrm{opt}$, an estimator $g$ with cost $T(g)$ can perform $K \approx T_\mathrm{opt} / T(g)$ iterations. Inserting this value for $K$ in eq. (\ref{eq:GK}) results in a convergence guarantee of
\begin{equation} \textstyle \Theta(\lambda, L, C, K, G^2(g)) = \frac{\alpha(\lambda, L, C)}{T_{\mathrm{opt}}^p} \left(G^2(g) T(g) \right)^p. \label{eq:GT} \end{equation}
In equation (\ref{eq:GT}) the only dependence of $\Theta(\lambda, L, C, K, G^2)$ on the estimator used $g$ comes through the $(G^2(g) T(g))^p$ factor. Thus, the best convergence guarantee is achieved by using the estimator with the smallest value of $G^2 T$, \emph{regardless} of the objective's properties $(\lambda, L, C)$ and total optimization time ($T_\mathrm{opt}$). We call this the ``$G^2 T$ principle''.

\begin{tcolorbox}[breakable, enhanced]
  \textbf{$\mathbf{G^2 T}$ principle.} Given
\begin{itemize}[leftmargin=*] \denselist
  \item A target function characterized by $(\lambda, L, C)$, which can be convex, strongly convex, strongly convex and smooth, or just smooth (non-convex);
  \item An optimization time budget $T_{\mathrm{opt}}$;
  \item A pool of gradient estimators. 
\end{itemize}
Then, \textit{for any} $(\lambda, L, C, T_{\mathrm{opt}})$, provided the appropriate algorithm is used, the best convergence guarantee is obtained using the estimator $g$ with minimum $G^2 T$. That is, given two estimators $g_1$ and $g_2$, $K_1 = T_{\mathrm{opt}}/T(g_1)$ and $K_2 = T_{\mathrm{opt}}/T(g_2)$, we have
\[ \Theta\left(\lambda, L, C, K_1, G^2(g_1)\right) < \Theta\left(\lambda, L, C, K_2, G^2(g_2)\right) \]
\[\iff G^2(g_1) T(g_1) < G^2(g_2) T(g_2),\]
\emph{regardless} of the values of $\lambda$, $L$, $C$ and $T_{\mathrm{opt}}$.
\end{tcolorbox}

\subsection{$\mathbf{G^2 T}$ for Gradient Estimator Selection} \label{sec:estselec}

This section presents a gradient estimator selection algorithm based on the $G^2 T$ principle. Given a pool of estimators with known $G^2$ and $T$, the one with minimum $G^2 T$ should be used. In practice, however, $G^2$ and $T$ are typically not known. We propose to use estimates.

Assuming that the cost of an estimator $g(w, \xi)$ is independent of $w$, an estimate $\hat{T}(g)$ of $T(g)$ can be obtained for each $g\in\mathcal{G}$ through a single initial profiling phase.

Dealing with $G^2(g)$ is more involved. Convergence guarantees typically assume that $\E||g(w, \xi)||^2 \leq G^2(g)$ for \emph{all} $w$. Often (e.g. when $w$ is unbounded) this is not true for any finite $G^2(g)$. Even if a finite $G^2(g)$ exists, the bound may be conservative. In practice, optimization starting from $w_0$ will only visit a restricted part of the parameter space, $\mathcal{W}_0$. In that case, it is sufficient to bound $\E ||g(w, \xi)||^2$ for this set of $w$ that may actually be encountered. Since $\E||g(w, \xi)||^2$ tends to decrease as optimization proceeds, we propose to use $G^2(g, w_0) = \E||g(w_0, \xi)||^2$ as an effective upper bound for $\E||g(w, \xi)||^2$ when $w\in\mathcal{W}_0$. With this effective upper bound, we select an estimator by finding the one with minimum $G^2(g, w_0) T(g)$.

After $k$ optimization steps, $G^2(g,w_k)$ will typically be less than $G^2(g,w_0)$. Starting from the current solution $w_k$, only a smaller part of the parameter space will be visited in the future, meaning a stronger bound for future gradients may be possible. We propose to use the new effective upper bound $G^2(g, w_k) = \E ||g(w_k, \xi)||^2$, and to select the estimator $g$ that now minimizes $G^2(g, w_k) T(g)$. We perform this update only a few times throughout training \footnote{In preliminary experiments, we tried computing the optimal estimator just once by minimizing $G^2(g, w_0) T(g)$, and using that $g$ for the full optimization. While this works, it often does not find the best estimator. This is likely because the initial solution $w_0$ is not so informative about the performance of gradient estimators closer to the solution.}. Finally, since $G^2(g, w_k) = \E ||g(w_k, \xi)||^2$ cannot be computed in closed form, we use a Monte Carlo estimate, $\hat{G}^2(g, w_k)$. The final algorithm is summarized in Algorithm \ref{alg:g2tpractice}.

\begin{algorithm}
\begin{small}
\caption{SGD with minimum $\hat{G}^2 \hat{T}$ estimator.}
\label{alg:g2tpractice}
\begin{algorithmic}[100]
\REQUIRE Set of estimators, $\mathcal{G}$.
\REQUIRE Total opt time and times to re-select estimator.
\REQUIRE Number of MC samples $M$.
\STATE For all $g \in \mathcal{G}$ measure time $\hat{T}(g)$.
\FOR{$k = 1, 2, \cdots$}
\IF{time to re-select estimator}
\STATE $ \displaystyle{\forall g\in \mathcal{G}\text{, set }\hat{G}^2(g, w_k) = \frac{1}{M} \sum_{m=1}^M \Vert g(w_k,\xi_{km}) \Vert^2}$.
\STATE $\displaystyle{g \leftarrow \arg\min_{g \in \mathcal{G}} \hat{G}^2(g,w_k) \times \hat{T}(g)}$.
\ENDIF
\STATE $w_{k+1} = w_k - \eta_k\ g(w_k,\xi_k)$
\ENDFOR
\end{algorithmic}
\end{small}
\end{algorithm}

\subsection{Empirical Validation of the $\mathbf{G^2T}$ principle} \label{sec:g2tpractice}

This section presents an empirical validation of Algorithm \ref{alg:g2tpractice}. We consider the same setting as in Fig. \ref{fig:intro}. We ran SGVI on four different models: two variants of a Bayesian neural network, a hierarchical Poisson model, and Bayesian logistic regression. For the logistic regression model we used a Gaussian distribution with a full-rank covariance matrix as variational distribution, and for the other models we used a factorized Gaussian.\footnote{More details on the models and optimization parameters in Section \ref{sec:exp_details}} We use three gradient estimators:
\begin{itemize} \denselist
\item (Rep) Plain reparameterization estimator, which computes the entropy in closed form (eq. (\ref{eq:rep})),
\item (Miller) Estimator proposed by Miller et al. \cite{traylorreducevariance_adam}, which adds a Taylor expansion control variate to reduce (Rep)'s variance (see Section \ref{sec:CVSVI}). As suggested by Miller et al. we use this control variate with a fixed weight of $1$.
\item (STL) ``Sticking-the-landing'' estimator \cite{stickingthelanding}, which removes the last term of equation (\ref{eq:stl}) (it is always zero) and estimates the remaining terms with reparameterization.
\end{itemize}
Additionally, we use Algorithm \ref{alg:g2tpractice} with the set of estimators $\mathcal{G} = \{\mathrm{(Rep)}, \mathrm{(Miller)}, \mathrm{(STL)}\}$, which uses the estimator $g\in\mathcal{G}$ with minimum $\hat{G}^2 \hat{T}$. We use $M = 400$ and we re-select the optimal estimator three times throughout training, initially, after 10\% of training is done, and after 50\% of training is done. Results are shown in Fig. \ref{fig:alg1val}. For all models the use of Algorithm \ref{alg:g2tpractice} leads to results that are at least as good as the results obtained using the best estimator chosen with hindsight.

\begin{figure}[t]
  \centering
  \includegraphics[width=0.49\linewidth]{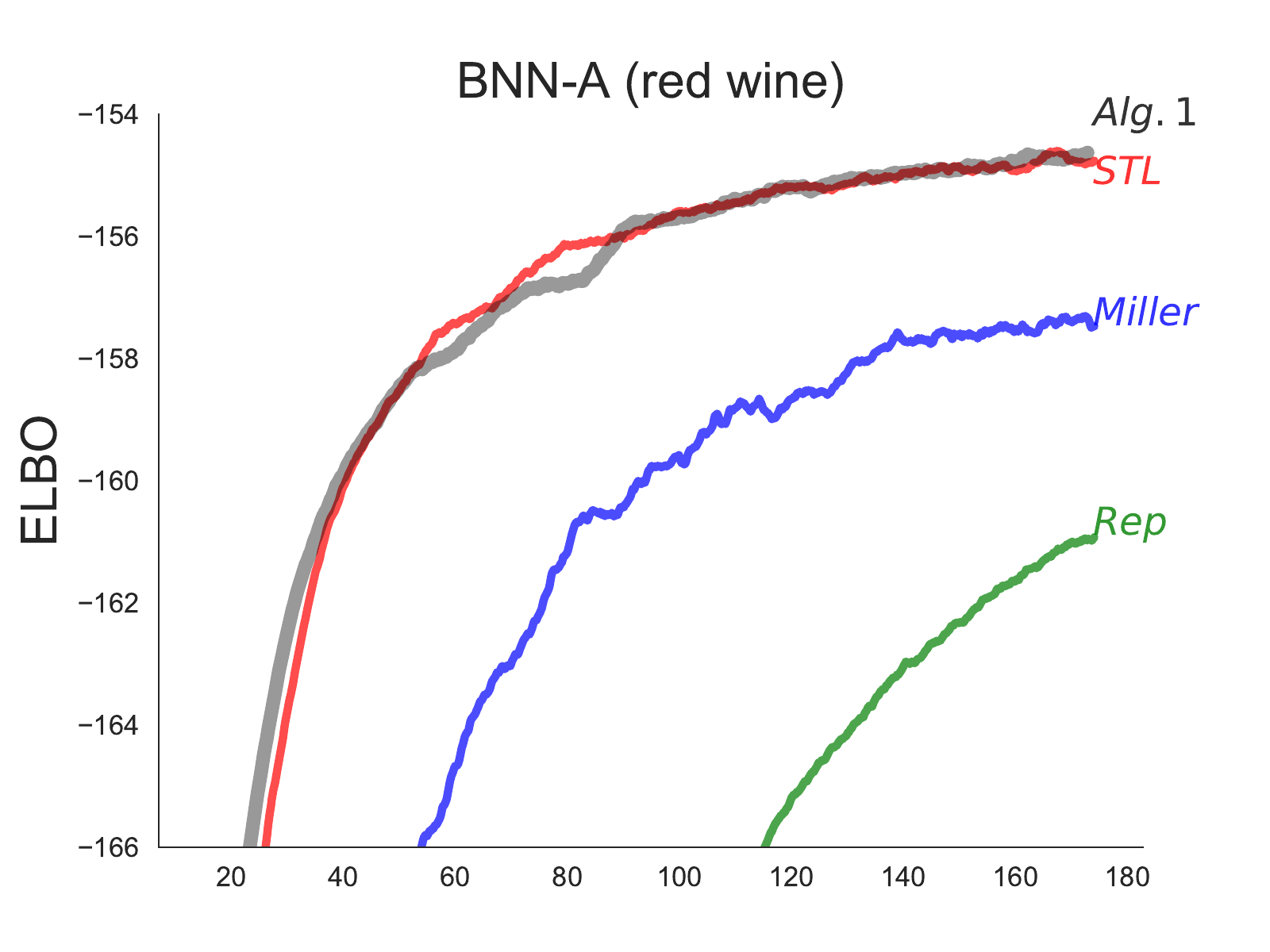}
  \includegraphics[width=0.49\linewidth]{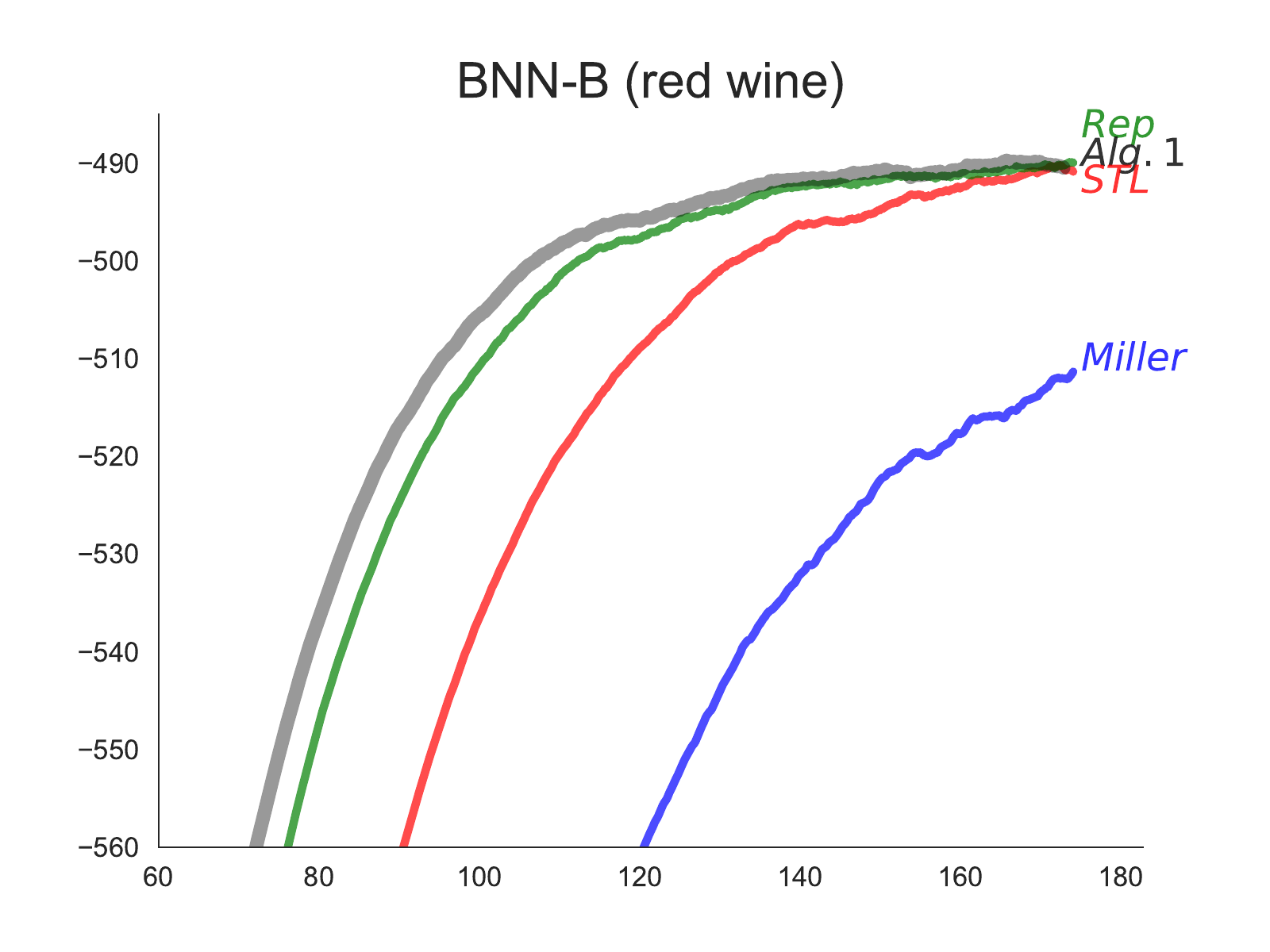}

  \includegraphics[width=0.49\linewidth]{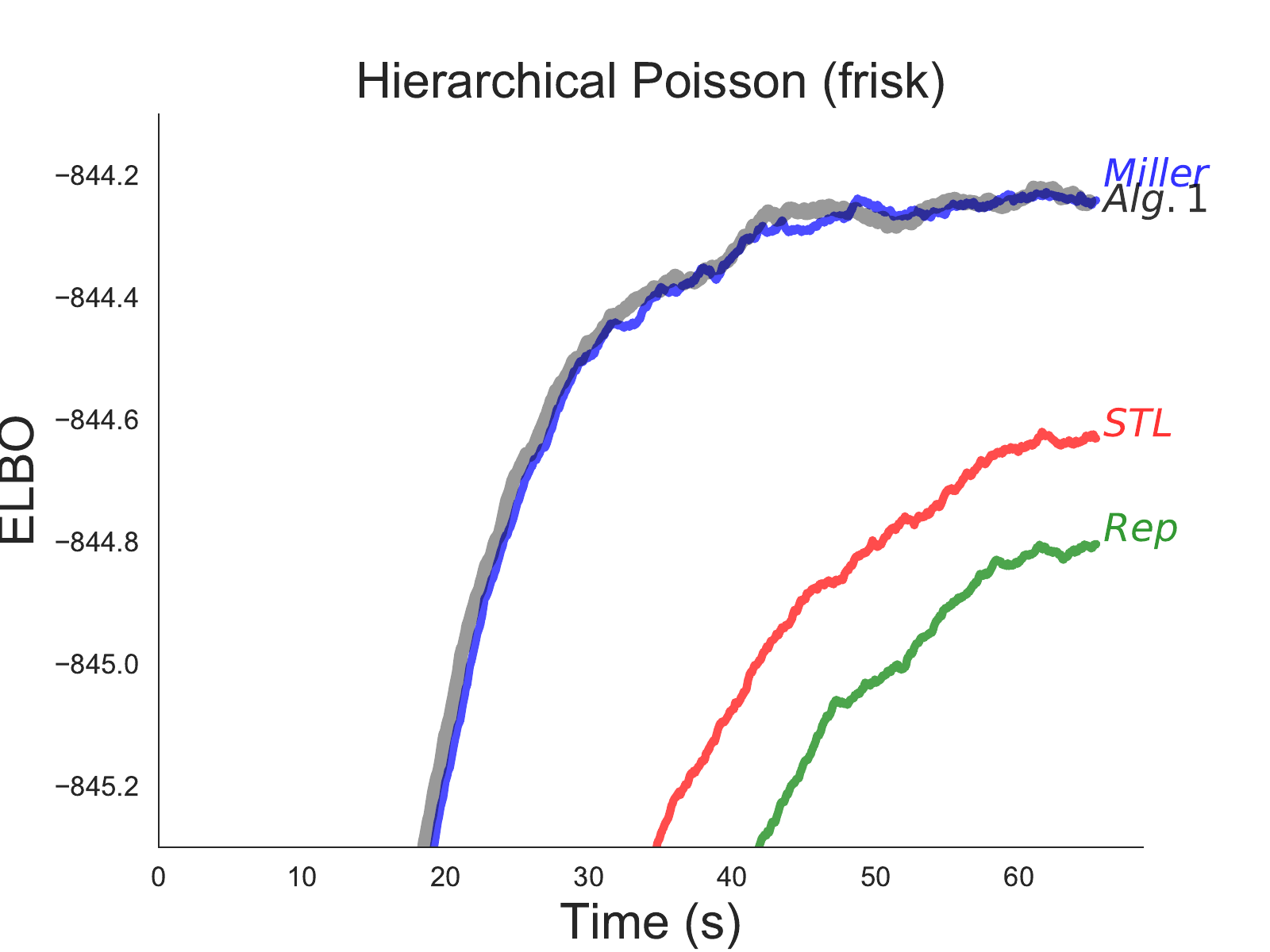}
  \includegraphics[width=0.49\linewidth]{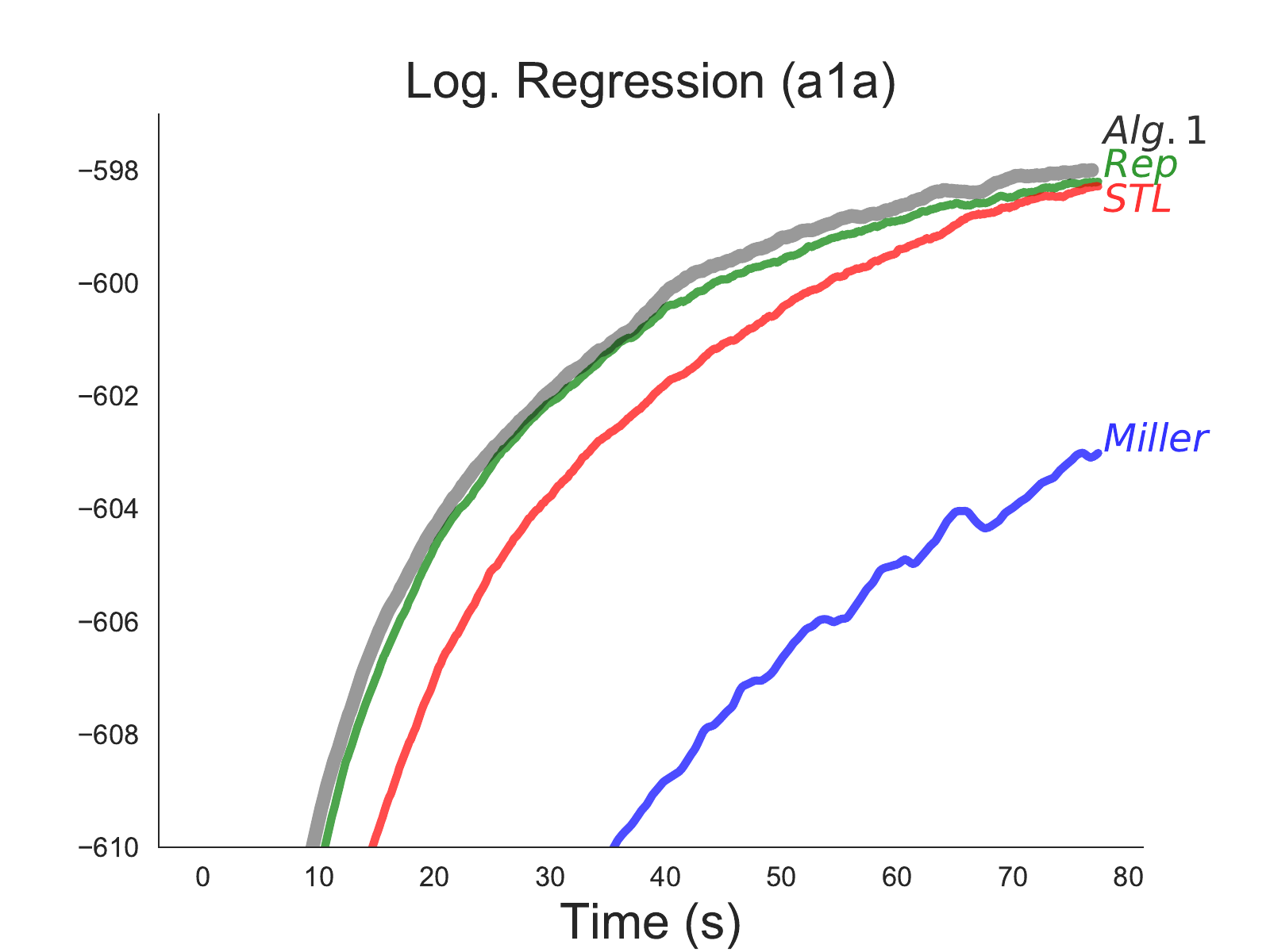}
  \caption{\textbf{Algorithm \ref{alg:g2tpractice} leads to results comparable to the results obtained using the best estimator chosen retrospectively.} SGVI on four models. Each plot compares Algorithm \ref{alg:g2tpractice} with three different gradient estimators, as described in the text.}
  \label{fig:alg1val}
\end{figure}

%% file: sections/miqcp2.tex
\section{$\mathbf{G^2T}$ with Control Variates} \label{sec:miqcp}

It is possible to use multiple control variates to reduce a gradient estimator's variance \cite{CVs}. However, some control variates might be computationally expensive but only result in a small reduction in variance. It may be better to remove them and accept a noisier but cheaper estimator. We should select what control variates to use adaptively.

Estimators with control variates can be expressed as
\begin{align} g_a(w, \xi) & = g_{\base}(w, \xi) + \sum_{i = 1}^J a_i c_i(w, \xi) \nonumber\\
& =  g_{\base}(w, \xi) + C(w, \xi) a. \label{eq:gacv} \end{align}
The set of available estimators is $\mathcal{G} = \{g_a : a \in \mathbb{R}^J\}$. In this case, the number of estimators $g_a\in\mathcal{G}$ is infinite. Therefore, Algorithm \ref{alg:g2tpractice} cannot be used (since we cannot measure $\hat{T}$ and $\hat{G}^2$ for each estimator individually).



This section shows how to find the minimum $\hat{G}^2 \hat{T}$ estimator when different estimators are indexed by a set of control variate weights. This is possible because two properties hold: (i) $\hat{T}(g_a)$ and $\hat{G}^2(g_a, w)$ can be efficiently obtained for all estimators $g_a \in\mathcal{G}$ through the use of shared statistics (a finite number of evaluations of the base gradients and control variates); and (ii) Finding $g_a$ that minimizes $\hat{G}^2(g_a, w) \hat{T}(g_a)$ can be reduced to a Mixed Integer Quadratically Constrained Program (MIQCP), which can be solved quickly in practice.

An expression for  $\hat{T}(g_a)$ can be obtained by noticing that computing $g_a$ only requires computing the base gradient and the control variates with \emph{non-zero} weights. Then, for all $g_a \in \mathcal{G}$,
\begin{equation} \hat{T}(g_a) = \hat{T}(g_\base) + \sum_{i = 1}^J \hat{T}(c_i)\  \mathbf{1}[a_i \neq 0]. \label{eq:Tga} \end{equation}
Thus, we can compute $\hat{T}(g_a)$ for all $g_a \in \mathcal{G}$ only by profiling the base gradient and each control variate individually.

Similarly, $\hat{G}^2(g_a, w)$ is determined by the same set of base gradient and control variate evaluations, regardless of the value of $a$. 
Suppose that, at iteration $k$, we sample $\xi_{k1}, ..., \xi_{kM}$. Then, for all $g_a \in \mathcal{G}$,
\begin{equation} \hat{G}^2(g_a, w_k) = \frac{1}{M} \sum_{m=1}^M \Vert g_{\base}(w_k, \xi_{km}) + C(w_k, \xi_{km}) \, a \Vert^2. \label{eq:Gga} \end{equation}
Thus, we can compute $\hat{G}^2(g_a,w_k)$ for all $g_a \in \mathcal{G}$ using only $M$ evaluations of the base gradient $g_\base$ and each control variate $c_i$.

Equations (\ref{eq:Tga}) and (\ref{eq:Gga}) characterize the (estimated) cost and variance of the gradient estimator with weights $a$. We find the weights that result in the optimal cost-variance tradeoff by solving
\begin{equation} \textstyle a^*(w) = \underset{a\in\mathbb{R}^J}{\arg\min} \,\, \hat{G}^2(g_a, w) \times \hat{T}(g_a), \label{eq:g2ta}\end{equation}
where $\hat{T}(g_a)$ and $\hat{G}^2(g_a, w)$ are as in equations (\ref{eq:Tga}) and (\ref{eq:Gga}). The solution $a^*(w)$ indicates what control variates to use (those with $a_i^* \neq 0$), and their weights.

Solving the (combinatorial) minimization problem in equation (\ref{eq:g2ta}) may be challenging. However, theorem \ref{thm:opta} states that it can be reduced to a MIQCP, which can be solved fast using solvers such as Gurobi \cite{gurobi}.

\begin{restatable}{thm}{optimala} \label{thm:opta}
When different gradient estimators are indexed by a set of $J$ control variate weights, the problem of finding $a^*(w)$ as in equation (\ref{eq:g2ta}) can be reduced to solving a mixed integer quadratically constrained program with $2 J + 2$ variables, one quadratic constraint, and one linear constraint.
\end{restatable}

The proof idea is as follows: If one expands the norm, equation (\ref{eq:Gga}) can be written as a quadratic function $\hat{G}^2(g_a,w_k) = \frac{1}{2} a^\top Q a + r^\top a + u$, where $Q, r$ and $u$ depend on the $M$ evaluations of the base gradient $g_\base$ and control variates $C$. Using this, the minimization problem from equation (\ref{eq:g2ta}) can be expressed as
\begin{align}
\textstyle a^* & = \underset{a\in\mathbb{R}^J, b \in \{0, 1\}^J}{\arg\min} \left( \frac{1}{2} a^\top Q a + r^\top a + u \right) \times \left(t_0 + t^\top b\right) \nonumber \\
& \mathrm{s.t}\,\, b_i = \mathbf{1}[a_i \neq 0], \label{eq:MIQCPconv}
\end{align}
where $t_0=\hat{T}(g_\base)$, and $t_i=\hat{T}(c_i)$. Going from eq. (\ref{eq:MIQCPconv}) to a MIQCP is straightforward. Available solvers can solve the final MIQCP fast. The full proof and resulting MIQCP problem are shown in the appendix.

The final algorithm is summarized in Algorithm \ref{alg:g2tcv}.

\begin{algorithm}
\begin{small}
\caption{SGD with minimum $\hat{G}^2 \hat{T}$ estimator; estimators indexed by control variates weights.}
\label{alg:g2tcv}
\begin{algorithmic}[100]
\REQUIRE Set of available gradient estimators $\mathcal{G}$.
\REQUIRE Total opt time and times to re-select estimator.
\REQUIRE Number of MC samples $M$.
\STATE For $g_{\mathrm{base}}$ measure time $t_0$.
\STATE For all $i = 1,...,J$, measure $c_i$ time $t_i$.
\FOR{$k = 1, 2, \cdots$}
\IF{time to re-select estimator}
\STATE $a = \mathrm{argmin}_a\hat{G}^2(g_a, w_k) \hat{T}(g_a)$ (solve MIQCP).
\ENDIF
\STATE $g_a = g_\mathrm{base}(w, \xi_k) + \sum_{i:a_i\neq0} a_i c_i(w, \xi_k)$
\STATE $w_{k+1} = w_k - \eta_k\ g_a(w_k,\xi_k)$
\ENDFOR
\end{algorithmic}
\end{small}
\end{algorithm}

%% file: sections/experiments2.tex

\section{Experiments and Results}
\label{sec:exps}

This section presents results that empirically validate Alg. \ref{alg:g2tcv} for control variate selection. We tackle variational inference problems on several probabilistic models using SGVI. We perform stochastic gradient momentum while using Algorithm \ref{alg:g2tcv} to select the gradient estimator at three different times throughout optimization. The set of candidate estimators is $\mathcal{G}_\mathrm{Auto}=\{g_a | a \in \mathbb{R}^3 \}$, where $g_a$ is as defined in eq. (\ref{eq:gacv}), the base estimator is plain reparameterization, and there are three candidate control variates ($c_1, c_2, c_3$).

The goal is to check if Alg. \ref{alg:g2tcv} successfully navigates time / variance tradeoffs. We thus compare against using \emph{fixed} subsets of control variates with the weights that minimize the estimator's variance (which can be estimated efficiently \cite{CVs}). We consider each possible subset of control variates $S\subseteq\{c_1, c_2, c_3\}$. This is essentially the same as Algorithm \ref{alg:g2tcv}, but with the set of estimators $\mathcal{G}_S$, which contains all estimators that use the fixed subset of control variates $S$. Since all estimators in $\mathcal{G}_S$ use the same control variates, they all have the same cost. Therefore, finding the estimator with optimal cost variance tradeoff (minimum $\hat{G}^2 \hat{T}$) reduces to finding the estimator with minimum $\hat{G}$.





\subsection{Experimental details}
\label{sec:exp_details}

\textbf{Tasks and datasets:} We use SGVI on several probabilistic models: (i) Bayesian logistic regression with the datasets Mushrooms and a1a, with a standard normal prior on the weights; (ii) a hierarchical Poisson model with the Frisk dataset \cite{frisk}; (iii) a Bayesian neural network with the Red-wine dataset, with a scaled standard normal prior on the weights and biases; and (iv) a Bayesian neural network with the Red-wine dataset, with a hierarchical distribution as prior on the weights and biases. (The likelihoods for (iii) and (iv) are the same, they only differ in the prior used. We denote them by BNN-A and BNN-B, respectively.) Details for all models are given in the appendix.

\textbf{Variational distribution $q_w(z)$: } For the simpler logistic regression models we use a Gaussian distribution with a full-rank covariance matrix as variational distribution. For the other more complex models we use a factorized Gaussian (diagonal covariance matrix).

\textbf{Base gradient estimator:} As base gradient estimator, $g_\base$, we use what seems to be the most common estimator. We compute the entropy term \small$\nabla_w \E_{q_w}[\log q_w(Z)]$ \normalsize in closed form, and estimate the term \small$\nabla_w \E_{q_w}[\log p(x, Z)]$ \normalsize with reparameterization.\footnote{Using $\mathcal{T}_w(\xi) = \mu + D^{1/2} \xi$, where $\xi \sim \mathcal{N}(0, I)$, $\mu$ is the mean of $q_w$ and $D^{1/2}$ is the Cholesky factorization of the covariance of $q_w$.}

\textbf{Control variates used:}
\begin{description} [leftmargin=*] \denselist
\vspace{-0.4cm}
  \item[$c_1$:] Difference between the entropy term computed exactly and estimated using reparameterization: \small$c(w, \xi) = \nabla_w \log q_w(\mathcal{T}_w(\xi)) - \nabla_w \E_{q_w} \log q_w(Z)$\normalsize.
  \item[$c_2$:] Control variate by Miller et al. \cite{traylorreducevariance_adam} based on a second order Taylor expansion of $\log p(x, z)$ (Sec \ref{sec:CVSVI}).
  \item[$c_3$:] Difference between the prior term computed exactly and estimated using reparameterization: \small$c(w, \xi) = \nabla_w \log p(\mathcal{T}_w(\xi)) - \nabla_w \E_{q_w} \log p(Z)$\normalsize.
\end{description}

\textbf{Algorithmic details:} For Alg. \ref{alg:g2tcv} we use $M = 400$ to estimate $\hat{G}^2$ (except for Logistic regression, where we use $M = 200$). We re-select the optimal estimator three times during training, initially, after 10\% of training is done, and after 50\% of training is done.

\textbf{Optimization details:} We use SGD with momentum ($\beta = 0.9$) with 5 samples $z \sim q_w(z)$ to form the Monte Carlo gradient estimates. For all models we find an initial set of parameters by optimizing with the base gradient for 300 steps and a fixed learning rate of $10^{-5}$. This initialization was helpful in practice because $w$ tends to change rapidly at the beginning of optimization. After this brief initialization, $G^2(g, w)$ tends to change much more slowly, meaning our technique is more helpful.


\subsection{Results}

\begin{figure*}[ht]
  \centering
  \includegraphics[width=0.29\linewidth]{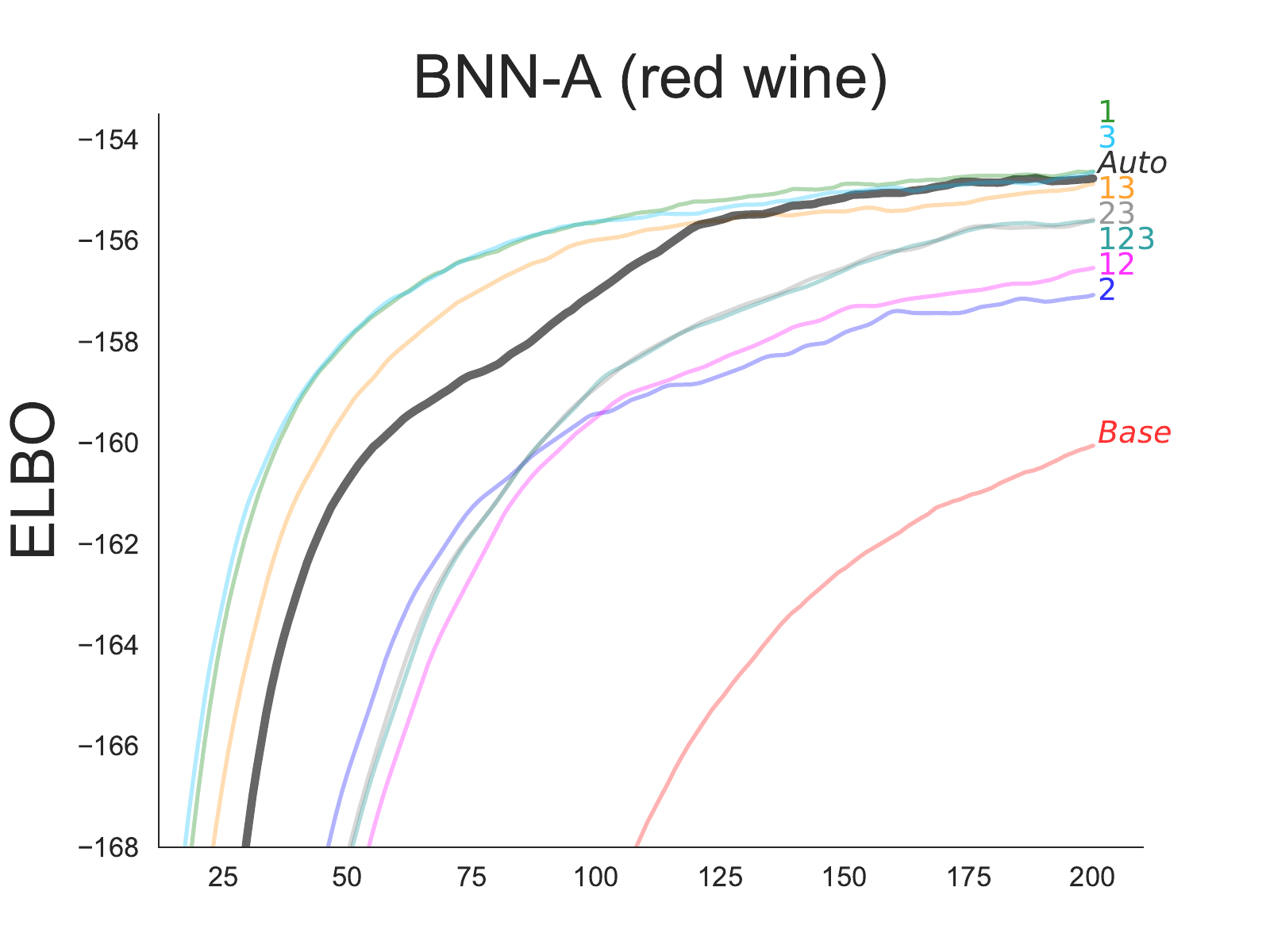}
  \includegraphics[width=0.29\linewidth]{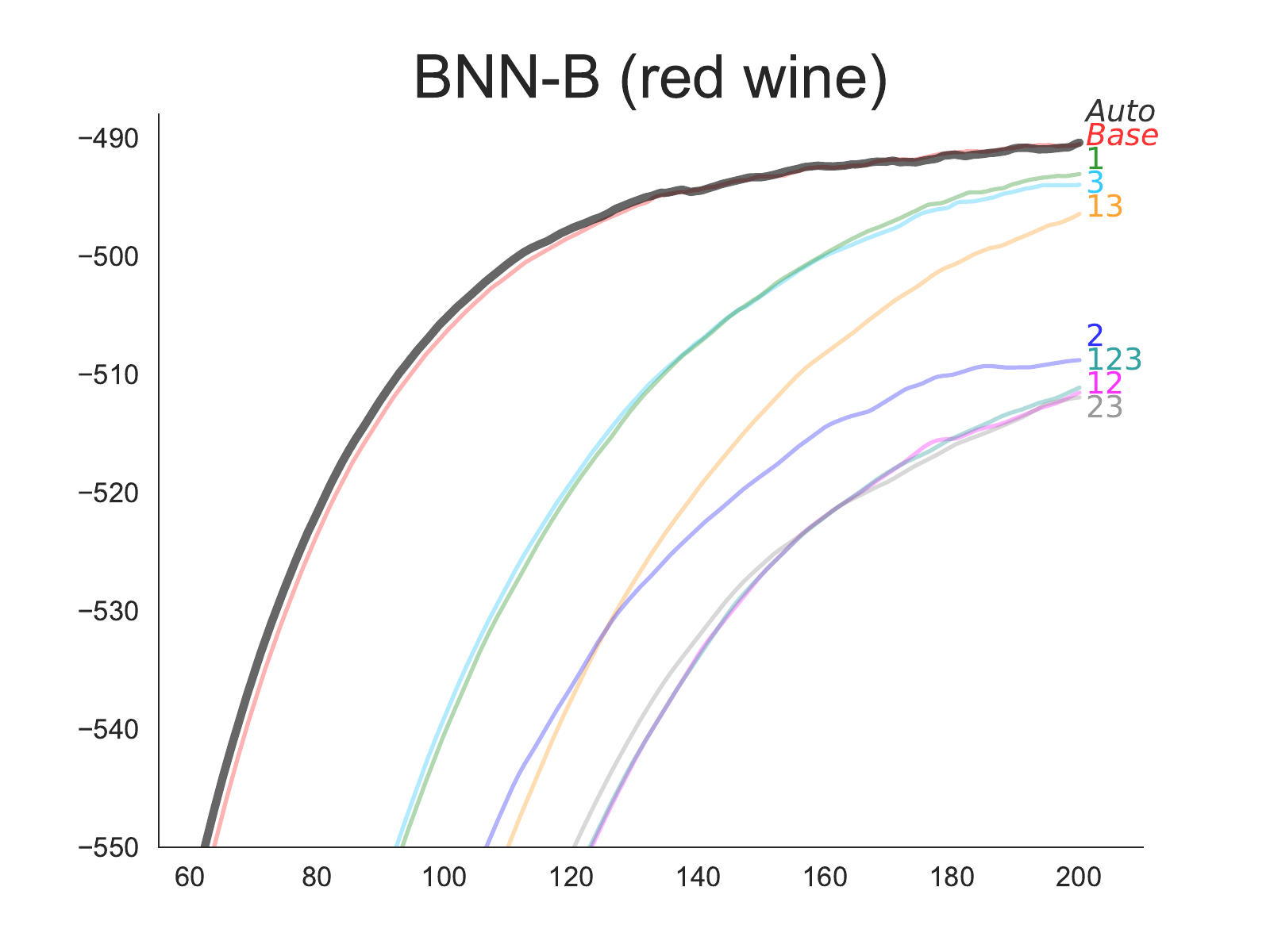}
  \includegraphics[width=0.29\linewidth]{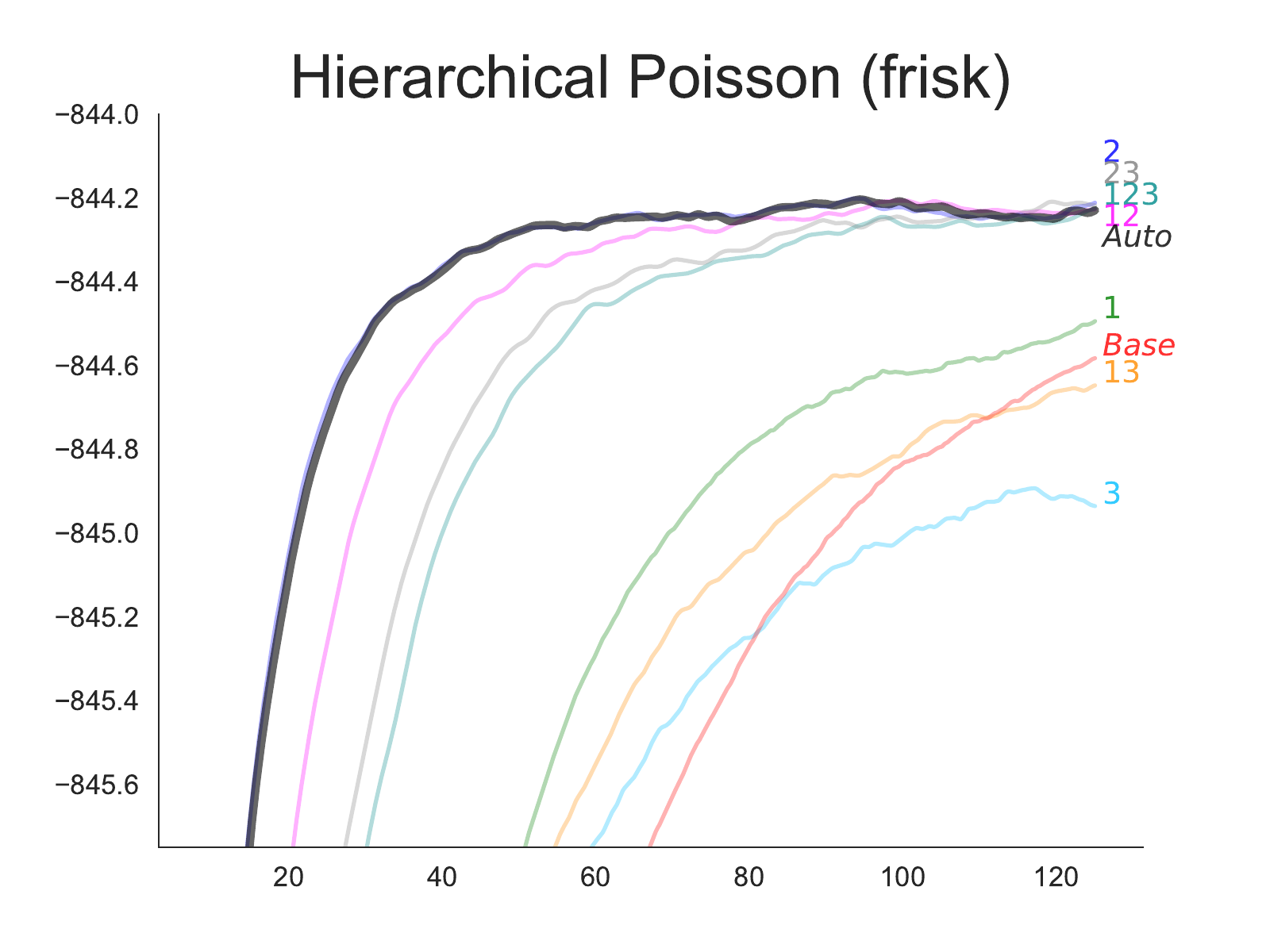}

  \includegraphics[width=0.29\linewidth]{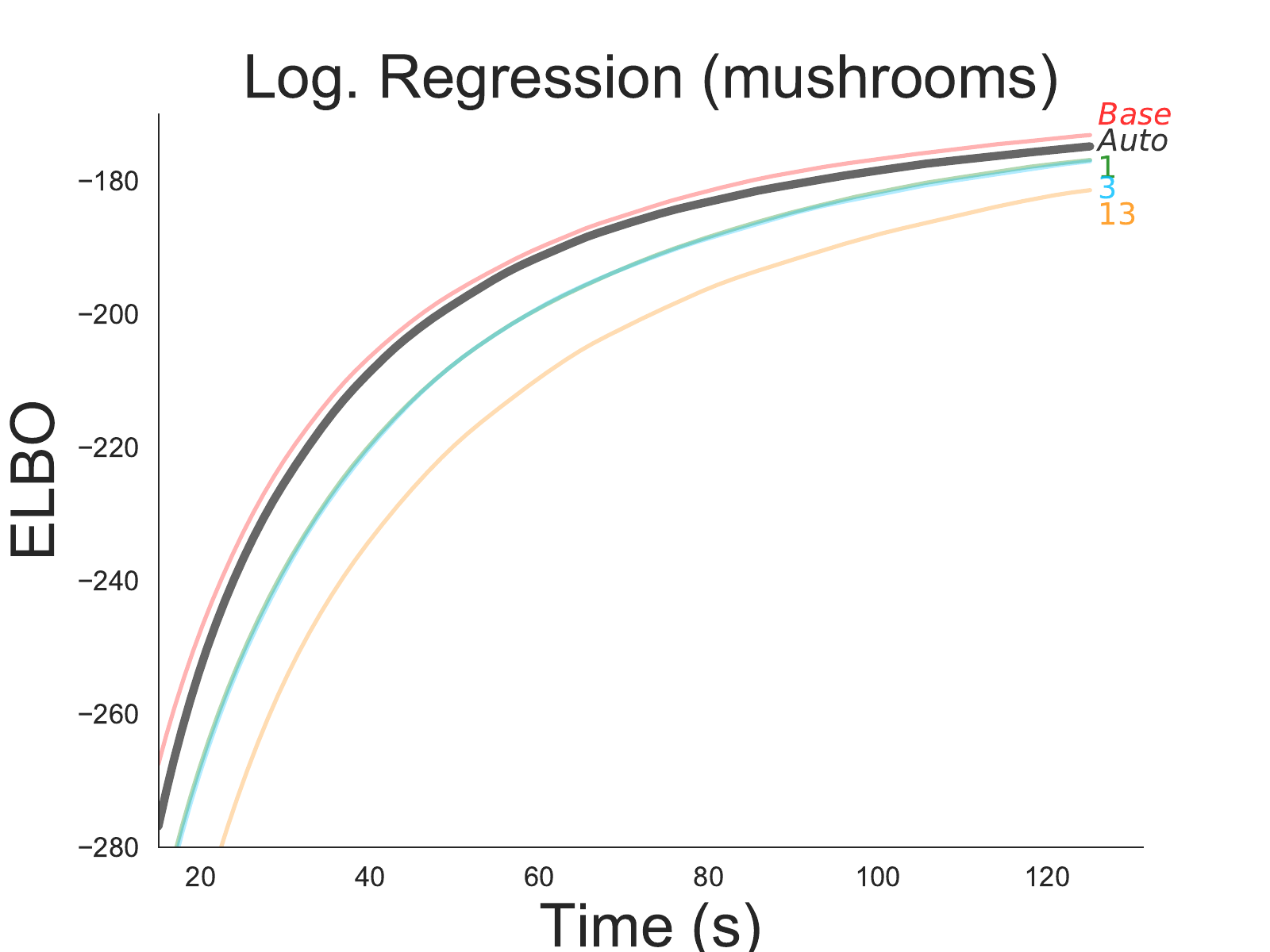}
  \includegraphics[width=0.29\linewidth]{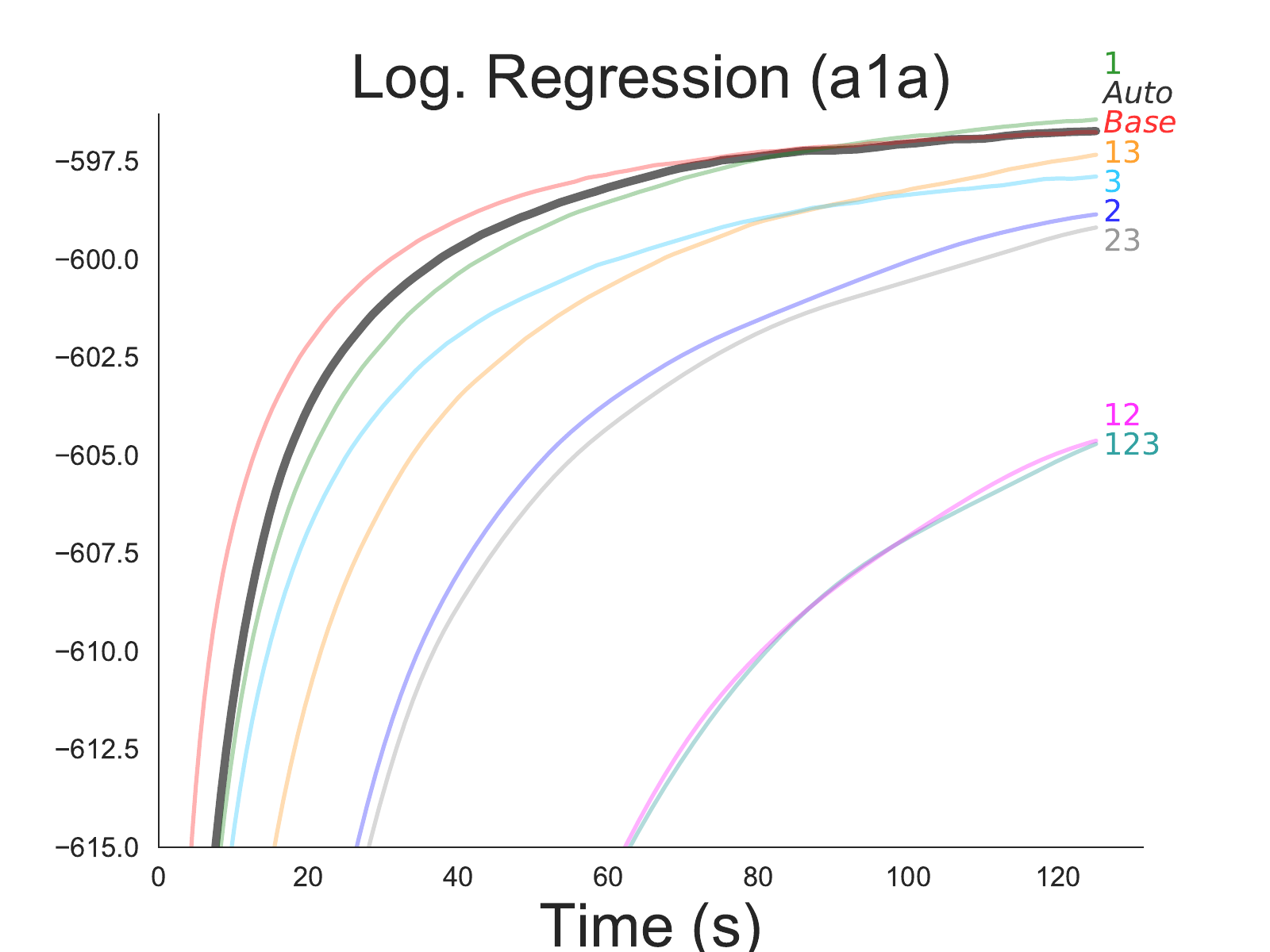}
  \caption{\textbf{Automatically selecting what control variates to use performs comparably to the best fixed set of control variates chosen with hindsight.} (Higher ELBO is better.) ``ELBO vs Time'' plots obtained using Alg. \ref{alg:g2tcv} to select what control variates to use and their weights (plotted in black). We compare against using different fixed subsets of control variates with the weights that minimize the estimator's variance. Five models considered, described in the text. Lines are identified as follows: ``Auto'' stands for using Alg. \ref{alg:g2tcv} to select what control variates to use and their weights, ``Base'' stands for optimizing using the base gradient alone, ``1'' stands for using the fixed set of control variates $\{c_1\}$ with the minimum variance weights, ``12'' stands for using the fixed set of control variates $\{c_1, c_2\}$, and so on. Overheads incurred by Algorithm \ref{alg:g2tcv} (profiling phase and MIQCP solving) are included in the times. (Note: For the model logistic regression (mushrooms) computing $c_2$ is slow. When $c_2$ is used, optimization does not converge in the amount of time provided.)}
  \label{fig:elbotime3}
\end{figure*}


Fig. \ref{fig:elbotime3} shows the optimization results (``ELBO vs. Time'' plots). All simulations were performed independently 20 times, and the average result is shown. The performance of all algorithms depends on the step-size. To give a fair comparison, Fig. \ref{fig:elbotime3} summarizes by showing the results with the best step-sizes.\footnote{12 stepsizes between $10^{-6}$ and $10^{-3}$ were considered.} It can be observed that automatically selecting what control variates to use and their weights (using Alg. 2) leads to good results for all models, as good as the ones obtained using the best fixed set of control variates chosen retrospectively. It can also be observed that results depend heavily on the model considered. For instance, for some models using the base gradient without any control variates appears to be the best choice (BNN-B), while for others it is suboptimal (BNN-A).

Raw results for individual step-sizes are shown in the appendix. The $G^2 T$ principle is not ideal when the step-sizes are chosen poorly. For example, if the step-size is chosen very small, convergence may be ``iteration limited'' rather than ``variance limited'', so that reducing variance of the gradient estimator does not speed up convergence. This is all consistent with our theory -- the convergence rates in Section \ref{sec:G2T} are based on well-chosen step-sizes.

Fig. \ref{fig:final} shows a summary of the results. For each model, it shows the final training ELBO achieved by using each possible fixed subset of control variates and the automatically chosen ones. It can be observed that there is no fixed subset of control variates that performs well across all models. However, using the automatically chosen control variates does.

\begin{figure}[ht]
  \centering
  \includegraphics[width=1.\linewidth]{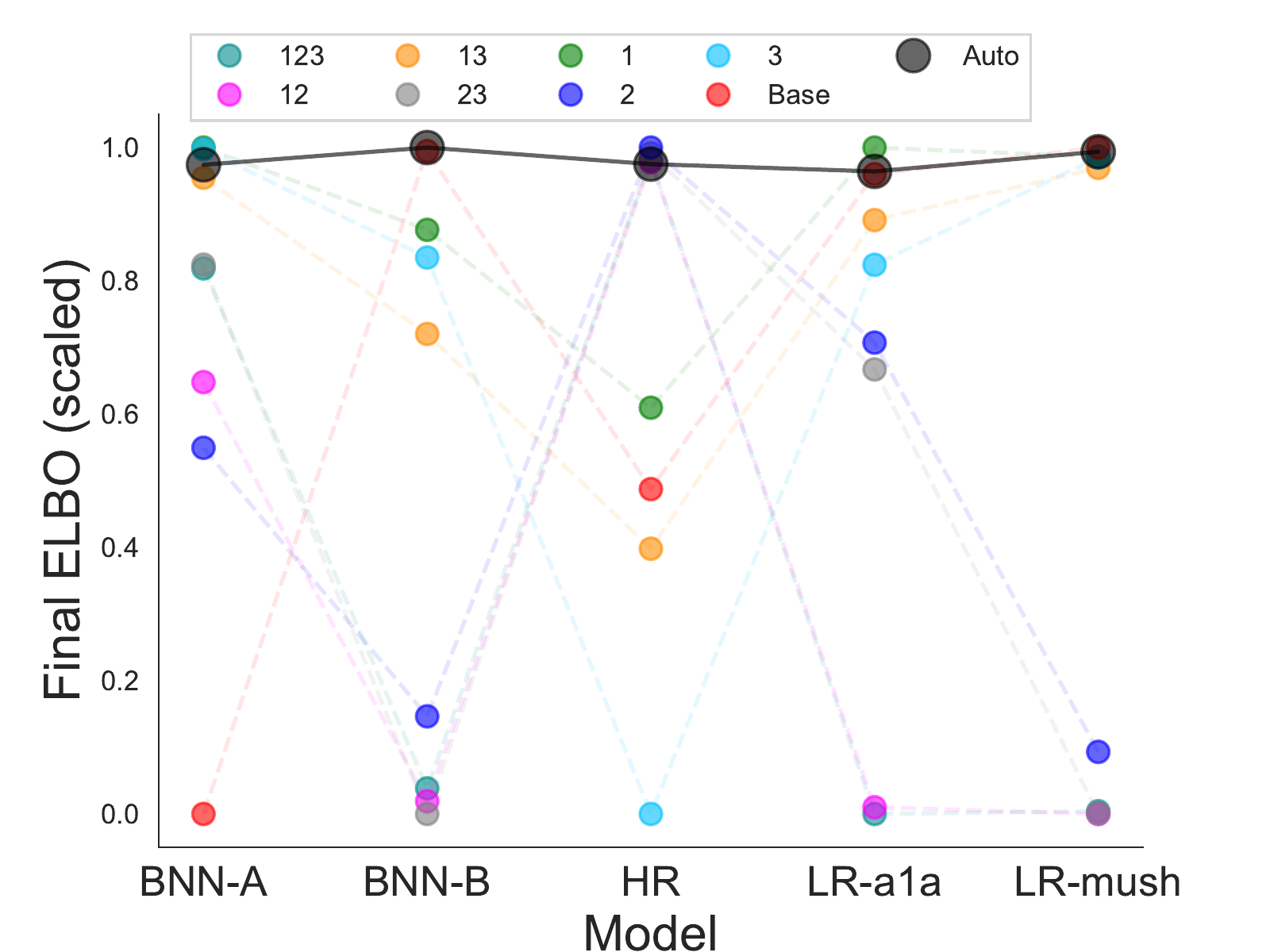}
  \caption{\textbf{The automatically chosen control variates perform well in all models.} (Higher ELBO is better.) For each model, we show the ELBO at the final time in Fig. \ref{fig:elbotime3}, with the lowest ELBO scaled to zero (bottom) and the highest scaled to one (top).}
  \label{fig:final}
\end{figure}

Figures \ref{fig:elbotime3} and \ref{fig:final} indicate that Alg. \ref{alg:g2tcv} selects a ``good'' set of control variates (and weights). What control variates are being chosen? This is shown in Table \ref{table:choices}. It can be observed that the algorithm makes different choices for different models, and that the control variates chosen depend on the stage of the training process (initially, or later stages). As shown in Fig. \ref{fig:elbotime3}, these choices lead to good results.

\renewcommand{\arraystretch}{1.1}
\begin{table}[ht]
\begin{small}
  \centering
  \begin{tabular}{llll}
    \toprule
    \multirow{2}{*}{\textbf{Model}} & \multicolumn{3}{c}{\textbf{Time of choice}} \\ \cmidrule{2-4}
    & $T = 0$ & $T = T_{opt}/10$ & $T = T_{opt}/2$\\ \midrule
    Log Reg (a1a) & $\{c_1, c_3\}$ & $\{c_1\}$ & $\{c_1\}$ \\
    Log Reg (mush) & $\emptyset$ & $\emptyset$ & $\{c_1\}$ \\
    Hier Poisson & $\{c_2\}$ & $\{c_2\}$ & $\{c_2\}$ \\
    BNN-A & $\{c_3\}$ & $\{c_1\}$ & $\{c_1\}$ \\
    BNN-B & $\emptyset$ & $\emptyset$ & $\emptyset$ \\
    \bottomrule
  \end{tabular}
  \caption{\textbf{Algorithm \ref{alg:g2tcv} selects different control variates for different models.} Control variates chosen by Algorithm \ref{alg:g2tcv}. Each column shows the set of control variates selected at some selection time. (In different runs different choices may ocurr. We shows the ``most popular'' choice across runs.)}
  \label{table:choices}
\end{small}
\end{table}
\vspace{-0.4cm}

%% file: sections/conclusions.tex
\section{Conclusion}




Results show that there is no single optimal estimator. The best estimator depends on several factors, such as model, dataset, and implementation. Estimator selection must be done adaptively. We presented the ``$G^2 T$ principle'' and used it to derive two gradient estimator selection algorithms. We showed empirically that both these algorithms work well in practice on inference problems tackled using SGVI.

Despite the empirical success of the proposed algorithms, several open questions remain: Can convergence guarantees use ``local'' $G^2$ values, as opposed to (looser) global bounds?
Should algorithms adapt (e.g. by changing the step size) when the underlying gradient estimator changes? We plan to address these questions in future work.

%% file: sections/appendix.tex

\section{Appendix}

\subsection{Details on models used}

Throughout the paper four different models are considered: two variants of a Bayesian neural network (BNN-A and BNN-B), a hierarchical Poisson model, and Bayesian logistic regression.

\textbf{Bayesian logistic regression: } We consider two different datasets:
``a1a'' and ``Mushrooms''. In all cases the training set is given by $\mathcal{D} = \{x_i, y_i\}$, where $y_i$ is binary. The model is specified by

\begin{align*}
w_i & \sim \mathcal{N}(0, 1),\\
p_i & = \left(1 + \exp (w_0 + w \cdot x_i)\right)^{-1},\\
y_i & \sim \mathrm{Bernoulli}(p_i).
\end{align*}

\textbf{Hierarchical Poisson model: } By Gelman et al. \cite{frisk}. The model measures the relative stop-and-frisk events in different precincts in New York city, for different ethnicities. The model is specified by

\begin{align*}
\mu & \sim \mathcal{N}(0, 10^2)\\
\log \sigma_\alpha & \sim \mathcal{N}(0, 10^2),\\
\log \sigma_\beta & \sim \mathcal{N}(0, 10^2),\\
\alpha_e & \sim \mathcal{N}(0, \sigma_\alpha^2),\\
\beta_p & \sim \mathcal{N}(0, \sigma_\beta^2),\\
\lambda_{ep} & = \exp(\mu + \alpha_e + \beta_p + \log N_{ep}),\\
Y_{ep} & \sim \mathrm{Poisson}(\lambda_{ep}).
\end{align*}

In this case, $e$ stands for ethnicity and $p$ for precinct, $Y_{ep}$ for the number of stops in precinct $p$ within ethnicity group $e$ (observed), and $N_{ep}$ for the total number of arrests in precinct $p$ within ethnicity group $e$ (observed).

\textbf{BNN-A: } As done by Miller et al. \cite{traylorreducevariance_adam} we use a subset of 100 rows from the ``Red-wine'' dataset (regression). We implement a neural network with one hidden layer with 50 units and Relu activations. Let $\mathcal{D} = \{x_i, y_i\}$ be the training set. The model is specified by

\begin{align*}
\log \alpha & \sim \mathcal{N}(0, 10^2),\\
\log \tau & \sim \mathcal{N}(0, 10^2),\\
w_i & \sim \mathcal{N}(0, \alpha^2), & \mbox{(weights and biases)}\\
\hat{y}_i & = \mathrm{FeedForward}(x_i, W),\\
y_i & \sim \mathcal{N}(\hat{y}_i, \tau^2).
\end{align*}

\textbf{BNN-B: } We use a subset of 200 rows from the ``Red-wine'' dataset (regression). We implement a neural network with one hidden layer with 50 units and Relu activations. Let $\mathcal{D} = \{x_i, y_i\}$ be the training set. The model is specified by

\begin{align*}
\log \tau & \sim \mathcal{N}(0, 5^2),\\
w_i & \sim \mathcal{N}(0, 5^2), & \mbox{(weights and biases)}\\
\hat{y}_i & = \mathrm{FeedForward}(x_i, W),\\
y_i & \sim \mathcal{N}(\hat{Y}_i, \tau^2).
\end{align*}

The only difference between BNN-A and BNN-B is in the prior used for the weights and biases.







\subsection{Mixed Integer Quadratic Program} \label{ap:MIQCP}

A mixed integer quadratic program is an optimization problem in which the objective function and constraints are quadratic (or linear), and some (or all) variables are restricted to be integers:
\begin{align}
\underset{x}{\mathrm{minimize}} \hspace{0.6cm} & \frac{1}{2} x^\top Q_0 x + r_0 ^\top x + u_0 \nonumber \\
\mbox{s.t.} \hspace{0.6cm} & \frac{1}{2} x^\top Q_i x + r_i ^\top x + u_i \geq 0 \hspace{0.5cm} i = 1, ..., m, \label{eq:origmiqcp}\\
& A x + b = 0, \nonumber
\end{align}
where $x \in \mathbb{R}^n$, $Q_0, ..., Q_m \in \mathbb{R}^{n\times n}$, and some components of $x$ are restricted to be integers.

We now prove Theorem \ref{thm:opta}.

\optimala*

\begin{proof}

Given
\begin{equation} \bar{T}(g_a) = \bar{T}(g_\base) + \sum_{i = 1}^J \bar{T}(c_i)\  \mathbf{1}[a_i \neq 0] \label{eq:pmt}\end{equation}
and
\begin{equation} \bar{G}^2(g_a, w) = \frac{1}{M} \sum_{m=1}^M \Vert g_{\base}(w, \xi_{m}) + C(w, \xi_{m}) a \Vert^2, \label{eq:pmg}\end{equation}
we want to find
\begin{equation} a^*(w) = \underset{a\in\mathbb{R}^J}{\arg\min} \,\, \bar{G}^2(g_a, w) \times \bar{T}(g_a). \label{eq:pma}\end{equation}

To simplify notation, we use $\bar{G}^2 = \bar{G}^2(g_a, w)$, $g_{bm} = g_{\base}(w, \xi_{m})$ and $C_m = C(w, \xi_m)$. Expanding the squared norm in eq. \ref{eq:pmg} we get
\begin{align}
\bar{G}^2 & = \frac{1}{M} \sum_{m=1}^M \Vert g_{bm} + C_m a \Vert^2 \nonumber\\
& = \frac{1}{M} \sum_{m=1}^M \left( g_{bm}^\top \, g_{bm} + 2 g_{bm}^\top C_m a + a^\top C_m^\top C_m a \right) \nonumber\\
& = \underbrace{\frac{1}{M} \sum_{m=1}^M \Vert g_{bm} \Vert^2}_{u_1} + \underbrace{\left( \frac{2}{M} \sum_{m=1}^M g_{bm}^\top C_m \right)}_{r_1} a\\
& + \frac{1}{2} a^\top \underbrace{\left( \frac{2}{M} \sum_{m=1}^M C_m^\top C_m \right)}_{Q_1} a \nonumber\\
& = u_1 + r_1^\top a + \frac{1}{2} a^\top Q_1 a. \label{eq:mg}
\end{align}

On the other hand, equation \ref{eq:pmt} can be expressed as
\begin{equation} \bar{T}(g_a) = t_0 + t^\top b, \,\, \mathrm{s.t.} \,\, b_i = \mathbf{1}[a_i \neq 0], \label{eq:mt} \end{equation}

where $t_0=\bar{T}(g_\base)$, and $t_i=\bar{T}(c_i)$. Using equations \ref{eq:mg} and \ref{eq:mt}, the minimization problem from equation \ref{eq:pma} can be expressed as
\small\begin{align} (a^*, b^*) & = \underset{a\in\mathbb{R}^J, b \in \{0, 1\}^J}{\arg\min} \left( \frac{1}{2} a^\top Q_1 a + r_1^\top a + u_1 \right) \times \left(t_0 + t^\top b\right), \nonumber\\
& \mathrm{s.t}\,\, b_i = \mathbf{1}[a_i \neq 0]. \label{eq:pmaa} \end{align}\normalsize

Introducing two extra varaibles, $V_G$ and $V_T$, we can express the minimization problem in eq. \ref{eq:pmaa} as 
\begin{align}
(a^*, b^*, V_G^*, V_T^*) & = \underset{a\in\mathbb{R}^J, b \in \{0, 1\}^J, V_G \in \mathbb{R}, V_T \in \mathbb{R} }{\arg\min} V_G \times V_T, \nonumber \\
& \mbox{s.t } V_G \geq \frac{1}{2} a^\top Q_1 a + r_1^\top a + u_1 \nonumber\\ 
& V_T = t_0 + t^\top b \nonumber \\
& b_i = \mathbf{1}[a_i \neq 0]. \label{eq:miqcpfinal}
\end{align}

The final minimization problem shown in equation \ref{eq:miqcpfinal} has the form of a general MIQCP, shown in equation \ref{eq:origmiqcp}, with the exception of the last constraint $b_i = \mathbf{1}[a_i \neq 0]$. Despite not being in the original definition of a MIQCP, several solver accept constraints of this type (Gurobi \cite{gurobi}, the solver used in our simulation, does).

\end{proof}

\subsection{SGD convergence rates} \label{ap:SGD}

Recall the following definitions:

\begin{description}
  \item[Convexity: ] $F$ is convex iff it $F(\theta x + (1 - \theta) y) \leq \theta F(x) + (1 - \theta) F(y) \,\, \forall \theta \in [0, 1]$.

  \item[Strong convexity: ] $F$ is $\lambda$-strongly convex iff $F(y) \geq F(x) + \nabla F(x)^\top (y - x) + \frac{\lambda}{2} ||y - x||^2 \,\, \forall x, y.$
  
  \item[Smoothness: ] $F$ has Lipschitz continuous gradient with constant $L$ ($F$ is $L$-smooth) iff $||\nabla F(y) - \nabla F(x)|| \leq L ||y - x|| \,\, \forall x, y.$\\
\end{description}

We now state the convergence rates presented in Table \ref{table:SGD} in more detail.

\begin{theorem}[F strongly convex; decaying step-size; By Rakhlin et al. \cite{shamir_SGDopt}] \label{thm:sc}
Let F be a $\lambda$-strongly convex function over a convex set $W$. If we assume $\E_\xi[||g(w, \xi)||^2] \leq G^2 \, \forall w\in W$ and set the step size $\eta_t = 1 / (\lambda t)$, then, after $K$ updates of SGD, we have
\[ \E \left[ ||w_K - w^*||^2 \right] \leq \frac{4}{\lambda^2} \frac{G^2}{K}; \,\, \mbox{ where } w^* = \mbox{argmin}_w F(w) \]
\end{theorem}

\begin{proof}
See Rakhlin et al. \cite{shamir_SGDopt}.
\end{proof}

\begin{corollary}[F strongly convex and smooth; decaying step-size; By Rakhlin et al. \cite{shamir_SGDopt}]\label{thm:scs}
Let F be an $L$-smooth and $\lambda$-strongly convex function over a convex set $W$. If we assume $\E_\xi[||g(w, \xi)||^2] \leq G^2 \, \forall w\in W$ and set the step size $\eta_t = 1 / (\lambda t)$, then, after $K$ updates of SGD, we have
\[ \E \left[ F(w_K) - F(w^*) \right] \leq \frac{2 L}{\lambda^2} \frac{G^2}{K}. \]
\end{corollary}

\begin{proof}
See Rakhlin et al. \cite{shamir_SGDopt}.
\end{proof}

\begin{theorem}[F convex; constant step-size; Nemirovski et al. \cite{nemirovski2009robust}]\label{thm:c}
Let F be a convex function over a convex set $W$. Assume $\E_\xi[||g(w, \xi)||^2] \leq G^2 \, \forall w\in W$. Then, after $K$ updates of SGD using the optimal learning rate $\eta^* = \frac{Dw}{G\sqrt{K}}$ we have
\[ \E \left[ F(\bar{w}) - F(w^*) \right] \leq D_w \frac{G}{\sqrt{K}}, \]

where $\bar{w} = \frac{1}{K} \sum_{i=1}^K w_i$ and $D_w = ||w_0 - w^*||$.
\end{theorem}

\begin{proof}
See Nemirovski et al. \cite{nemirovski2009robust}.
\end{proof}

\begin{theorem}[F smooth; constant step-size]\label{thm:c}
Let F be an $L$-smooth function. Assume $\E_\xi[||g(w, \xi)||^2] \leq G^2 \, \forall w$. Then, after $K$ updates of SGD using the optimal learning rate $\eta^* = \sqrt{\frac{2 D_f}{L K G^2}}$ we have
\[ \E \left[ \frac{1}{K} \sum_{i = 1}^K ||\nabla F(w_i)||^2 \right] \leq \sqrt{L D_f} \frac{G}{\sqrt{K}}, \]
where $D_f = \E[F(w_0) - F(w^*)]$.
\end{theorem}

\begin{proof}

This proof is a straightforward adaptation from the one by Bottou et al. \cite{bottou_SGDreview}.

\small
\begin{align*}
F(w_{t+1}) & \leq F(w_t) + \nabla F(w_t)^\top (w_{t+1} - w_t) + \frac{L}{2} ||w_{t+1} - w_t||^2 \\
		   & = F(w_t) - \eta_t \nabla F(w_t)^\top g_t + \frac{L}{2} ||\eta_t g_t||^2 \\
		   & = F(w_t) - \eta_t \nabla F(w_t)^\top g_t + \frac{L \eta_t^2}{2} ||g_t||^2.
\end{align*}
\normalsize

Taking the expectation on both sides we get

\small
\begin{align*}
\E[F(w_{t+1}) - F(w_t)] & \leq - \eta_t \nabla \E[F(w_t)^\top g_t] + \frac{L \eta_t^2}{2} \E ||g_t||^2 \\
						& = - \eta_t \E ||\nabla F(w_t)||^2 + \frac{L \eta_t^2}{2} G^2.
\end{align*}
\normalsize

If we take $\eta_t = \eta$, and sum up both sides of the inequality we get

\small
\begin{align*}
\sum_{t = 0}^{K-1} \E[F(w_{t+1}) - F(w_t)] & \leq - \eta \sum_{t = 0}^{K-1} \E ||\nabla F(w_t)||^2 + \frac{K L \eta^2}{2} G^2 \\
\E[F(w_{K}) - F(w_0)] & \leq - \eta \sum_{t = 0}^{K-1} \E ||\nabla F(w_t)||^2 + \frac{K L \eta^2}{2} G^2.
\end{align*}
\normalsize

Re-arranging and using the fact that $F(w^*) \leq F(w_K)$ gives

\small
\begin{align*}
\E \left[ \frac{1}{K} \sum_{t = 0}^{K-1} ||\nabla F(w_t)||^2 \right] & \leq \frac{L \eta}{2} G^2 + \frac{F(w_0) - F(w^*)}{\eta K} \\
																	 & = \frac{L \eta}{2} G^2 + \frac{D_f}{\eta K}.
\end{align*}
\normalsize

Where $D_f = F(w_0) - F(w^*)$. The value for $\eta$ that minimizes the right hand side of the last inequality is $\eta^* = \sqrt{\frac{D_f}{L K G^2}}$. Using $\eta = \eta^*$ yields

\small
\begin{align*}
\E \left[ \frac{1}{K} \sum_{t = 0}^{K-1} ||\nabla F(w_t)||^2 \right] & \leq \sqrt{\frac{L (F(w_0) - F(w^*)) G^2}{K}}.
\end{align*}
\normalsize

\end{proof}

Finally, slightly different versions of the last two bounds in Table \ref{table:SGD} were proposed by Yang et al. \cite{yang_momentum}. The authors carry out a unified analysis for stochastic momentum methods using a parameter $s$; if $s=0$ the algorithm results in the heavy ball method, and if $s=1$ in a stochastic variant of Nesterov's accelerated gradient. They state the following result:

\begin{theorem}[F smooth; constant step-size; stochastic momentum methods; Yang et al. \cite{yang_momentum}]
Let F be a (possibly non convex) $L$-smooth function, Assume $\E_\xi[||g(w, \xi) - \nabla F(w)||^2] \leq \delta^2$ and $||\nabla F(w)||^2 \leq V^2$. Then, after $K$ updates of the proposed stochastic momentum method ($\beta$) with learning rate $\eta = \mathrm{min}\left \{ \frac{1-\beta}{2L}, \frac{C}{\sqrt{T}} \right\}$, we have
\[\underset{k=0,...,K}{\min} \E ||\nabla F(w_k)||^2 \leq \frac{2 D_f (1-\beta)}{T} \mathrm{max} \left\{ \frac{2L}{1-\beta}, \frac{\sqrt{T}}{C} \right\}\] 
\[+ \frac{C}{\sqrt{T}} \frac{L \beta^2 \left( (1 - \beta)s - 1 \right)^2 (V^2 + \delta^2) + L \delta^2 (1-\beta)^2}{(1-\beta)^3},\]
where $D_f = F(w_0) - F(w^*)$. 
\end{theorem}

The results shown in Table \ref{table:SGD} are obtained by: 1) setting $s=0$ (momentum) or $s=1$ (Nesterov); 2) finding the optimal $C$; 3) bounding $\E_\xi ||g(w, \xi)||^2 \leq \delta^2 + V^2 = G^2$; and 4) assuming that optimization is performed for a large enough number of steps $\left(K \geq \frac{4L^2 C^2}{1-\beta}\right)$.



\subsection{Raw results for individual step-sizes}

\clearpage
\newpage


\begin{figure*}[ht]
  \centering
  \includegraphics[width=0.45\linewidth]{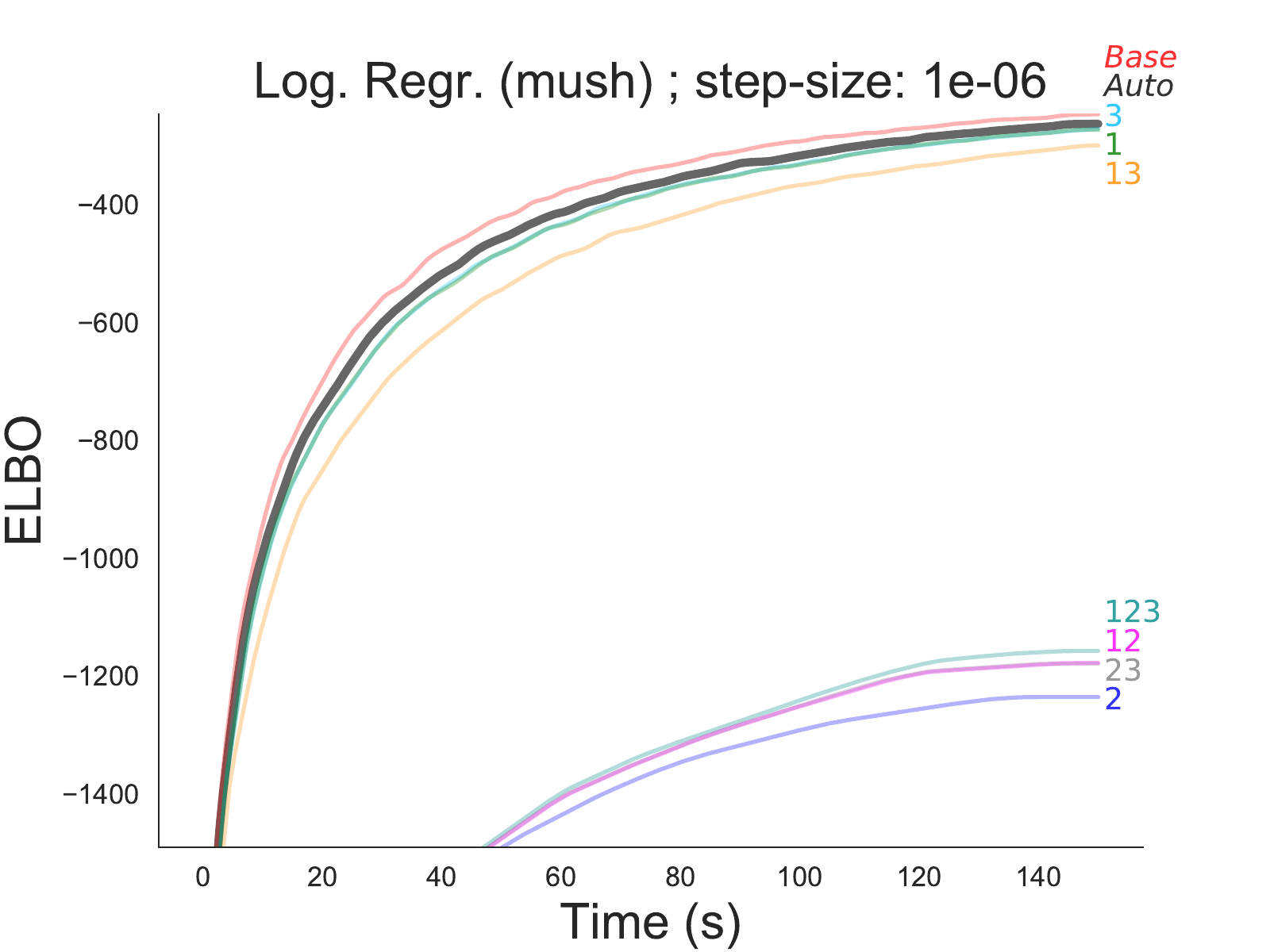}
  \includegraphics[width=0.45\linewidth]{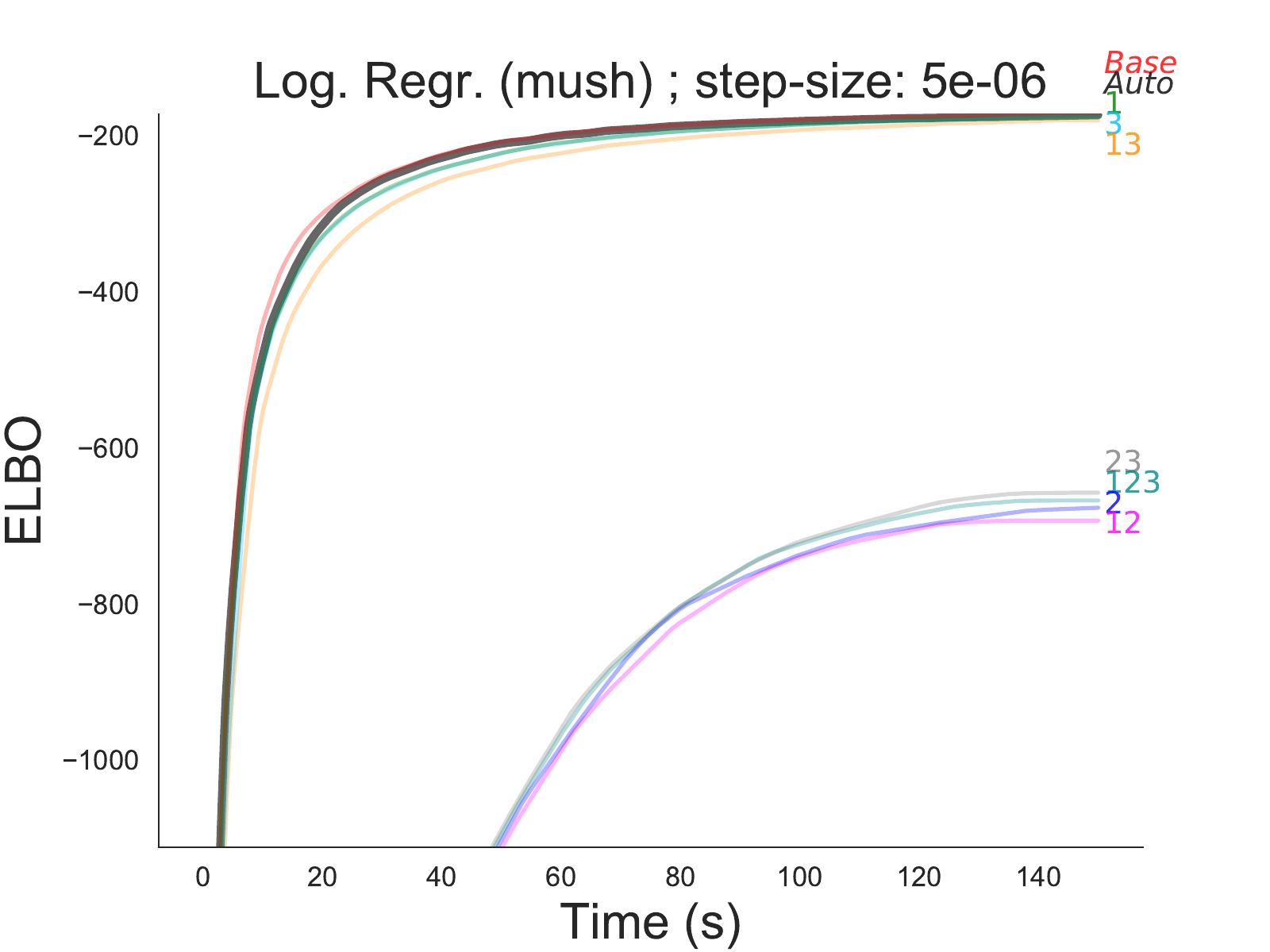}

  \includegraphics[width=0.45\linewidth]{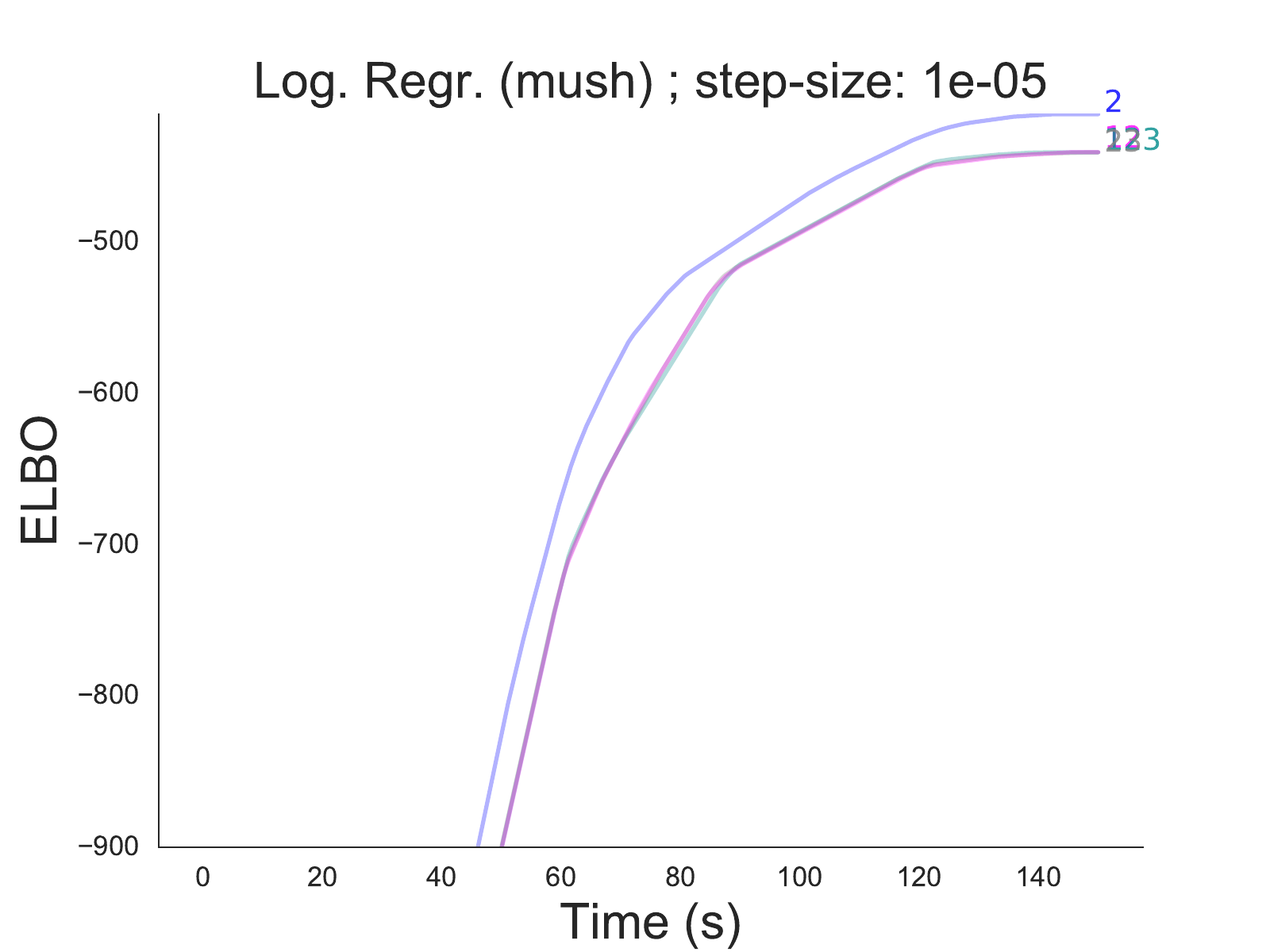}
\end{figure*}

\clearpage
\newpage


\begin{figure*}[ht]
  \centering
  \includegraphics[width=0.45\linewidth]{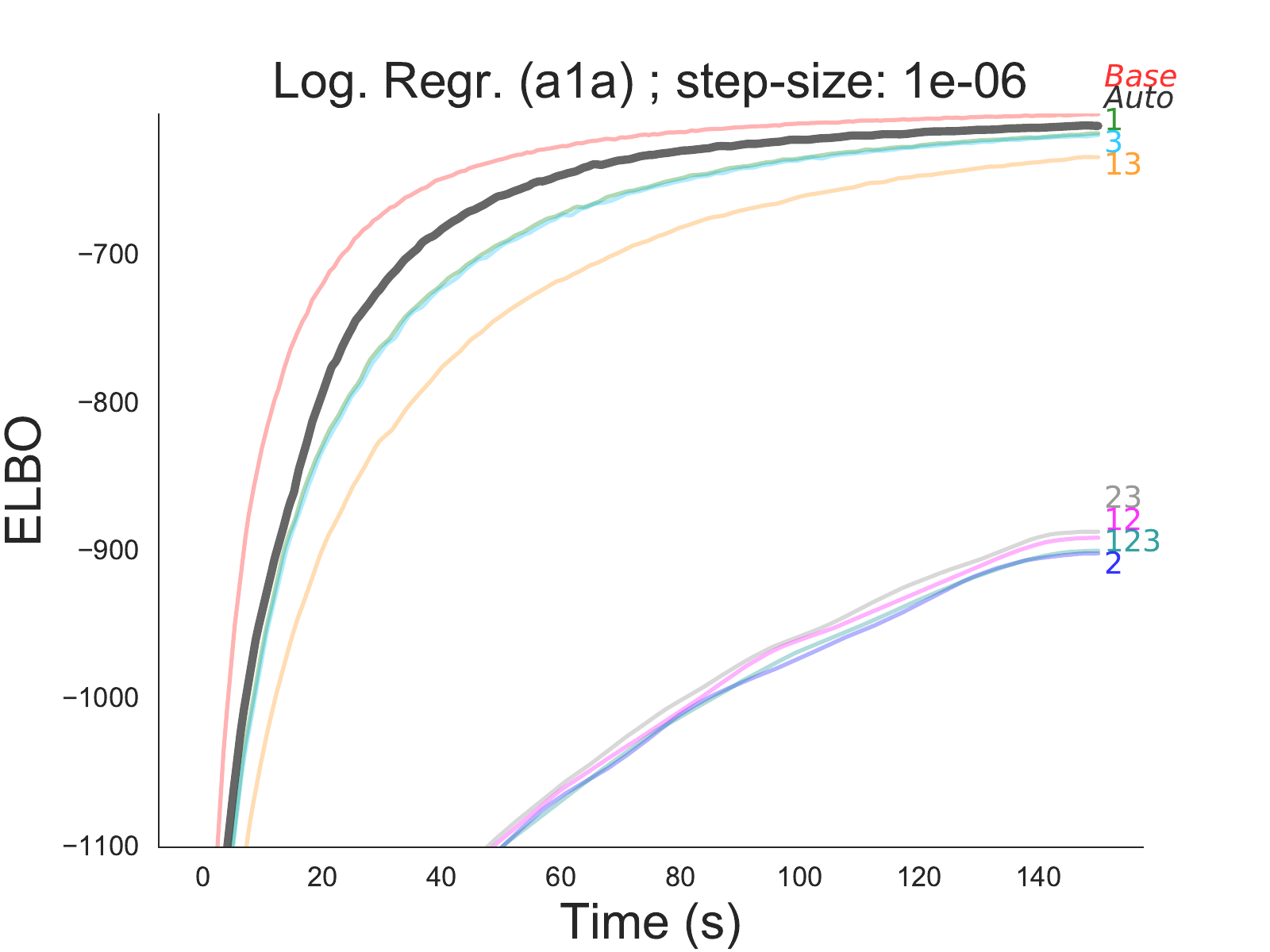}
  \includegraphics[width=0.45\linewidth]{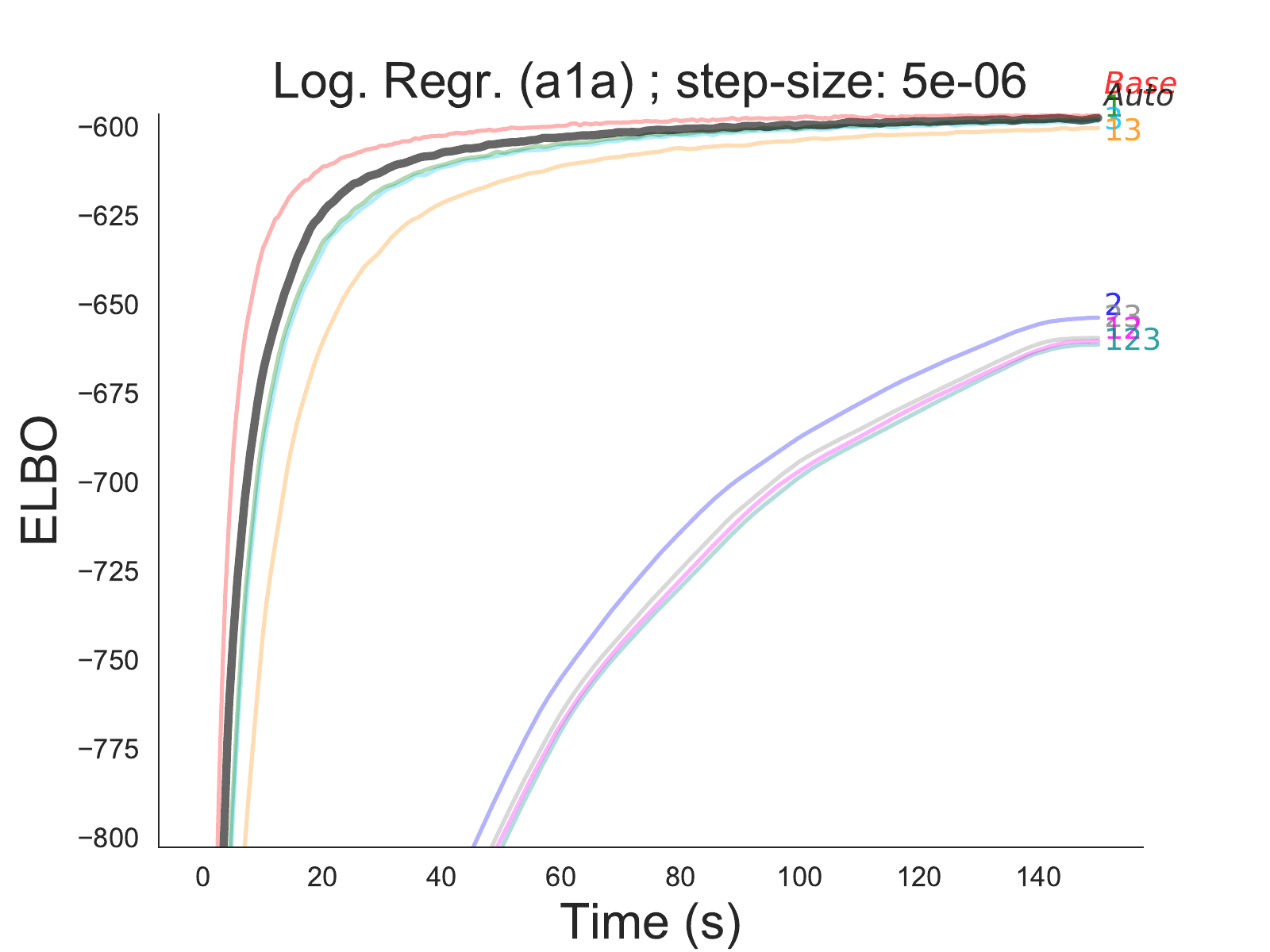}
  
  \includegraphics[width=0.45\linewidth]{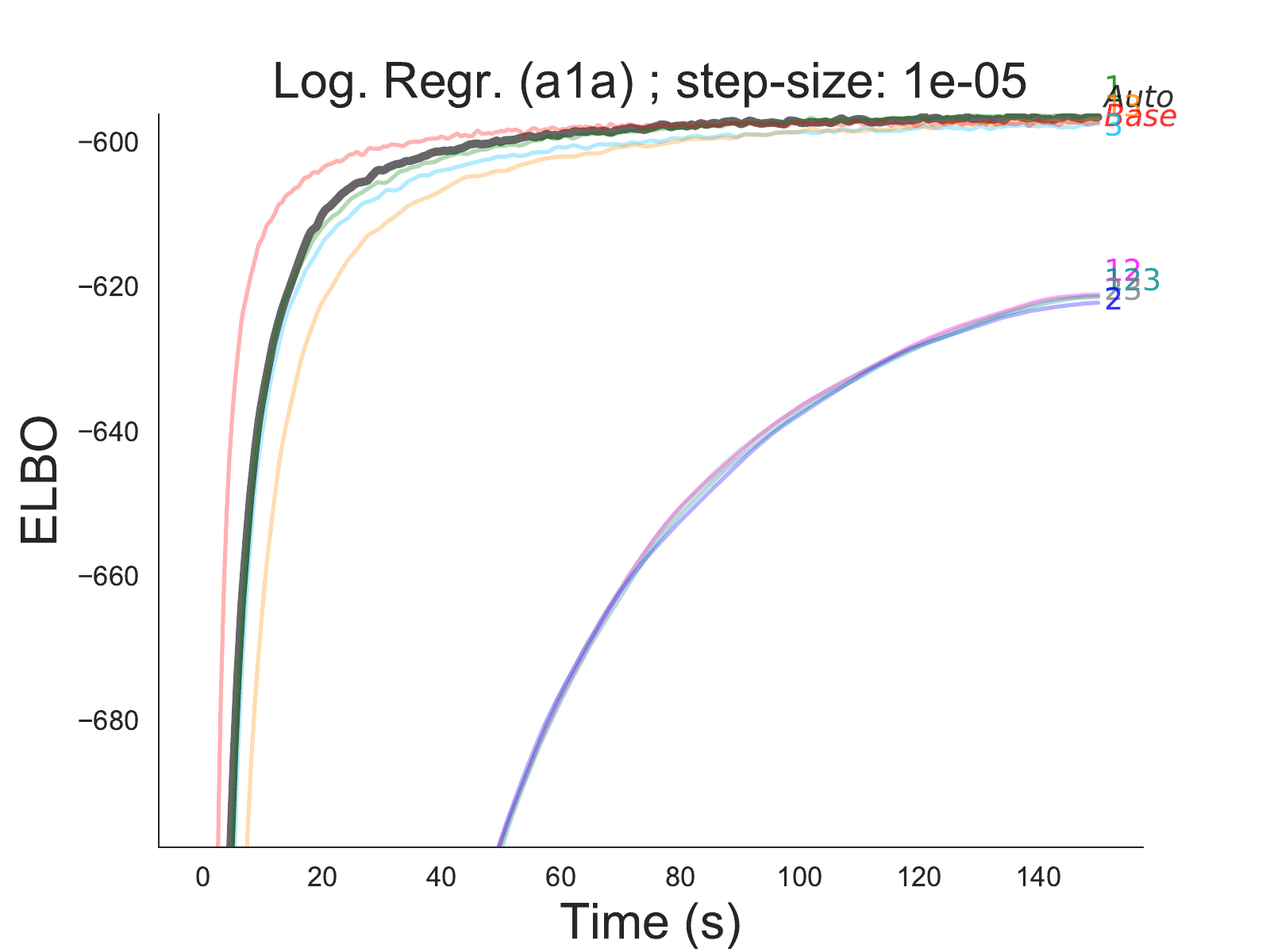}
  \includegraphics[width=0.45\linewidth]{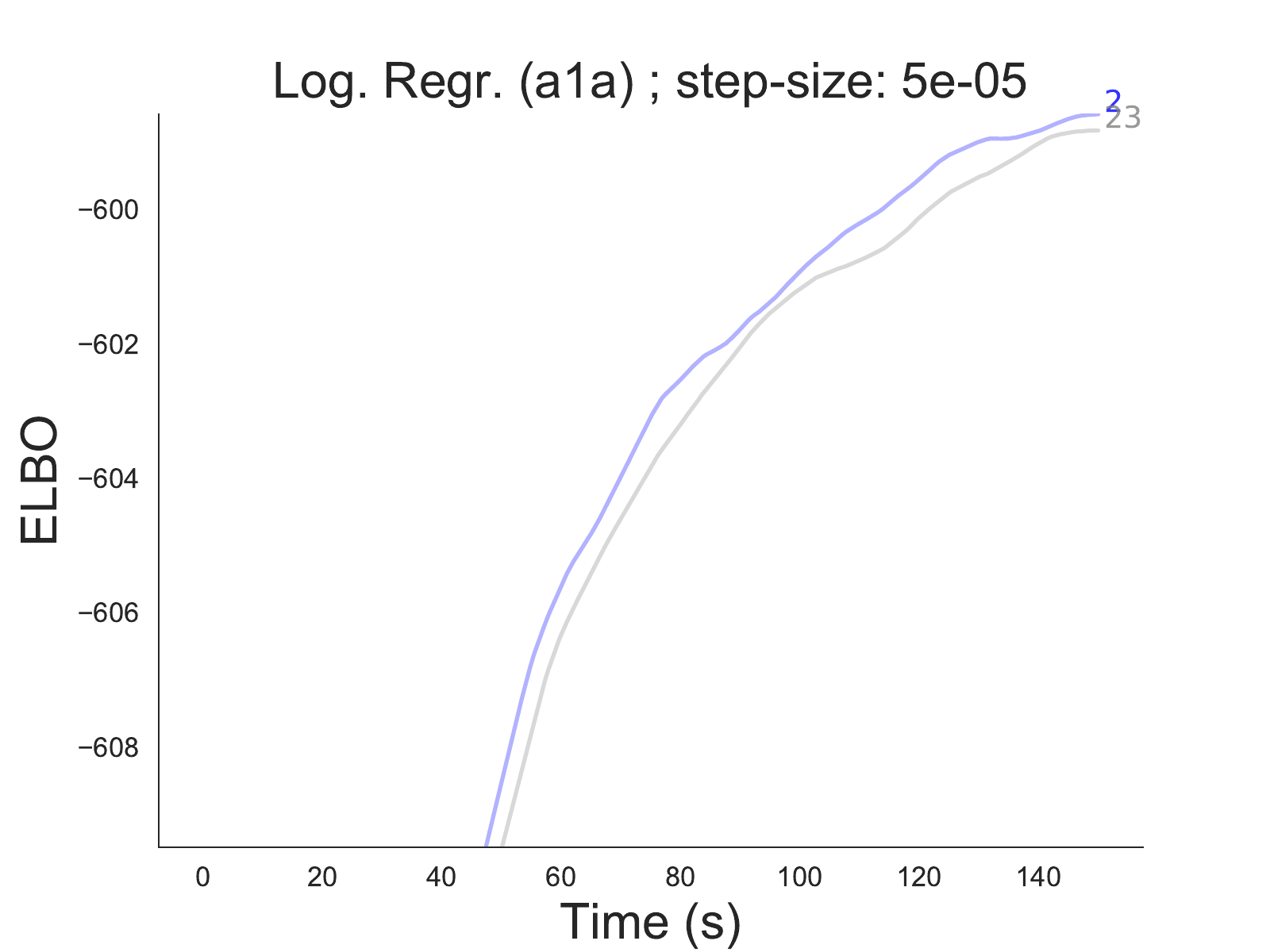}
\end{figure*}


\begin{figure*}[ht]
  \centering
  \includegraphics[width=0.45\linewidth]{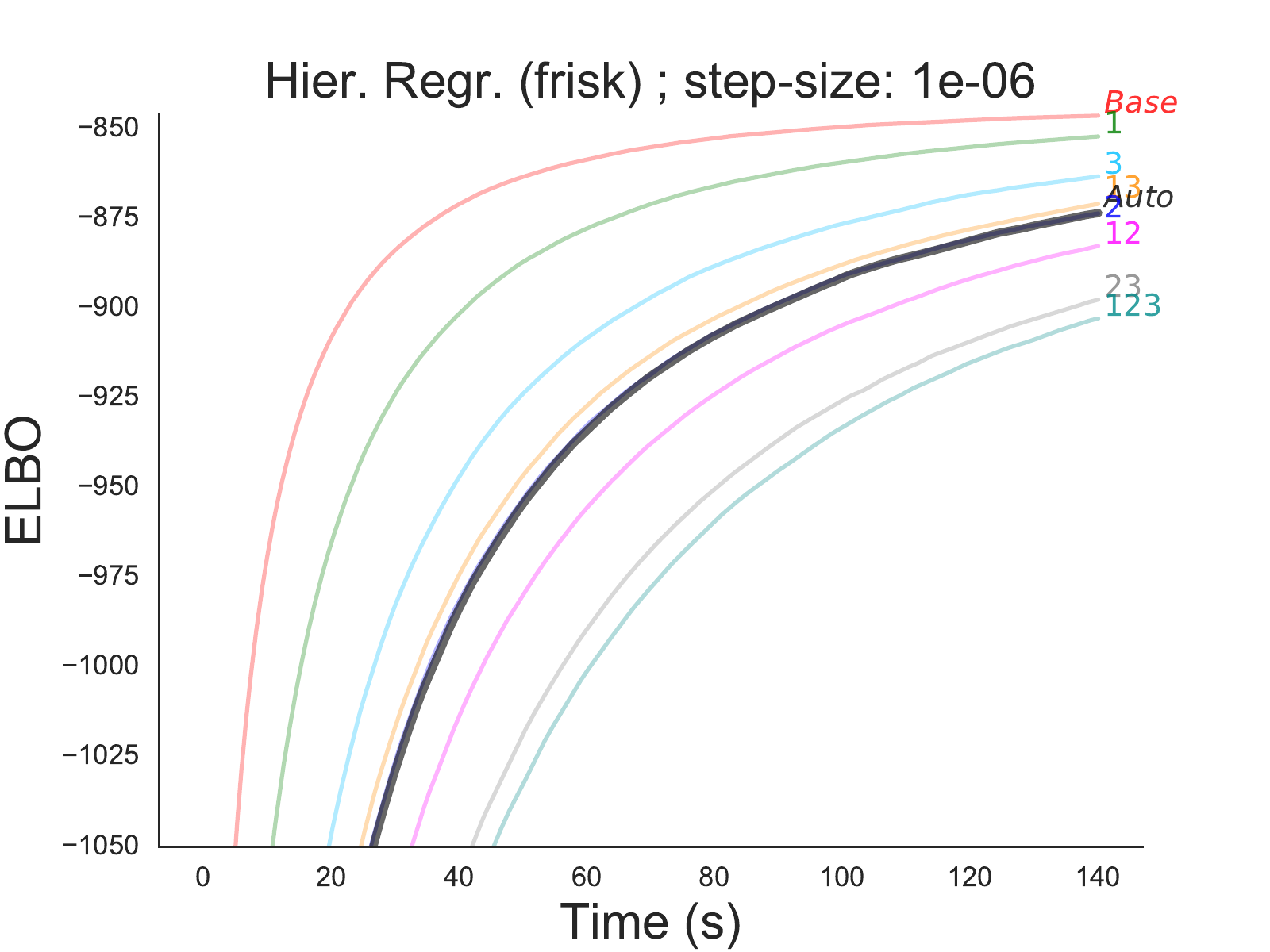}
  \includegraphics[width=0.45\linewidth]{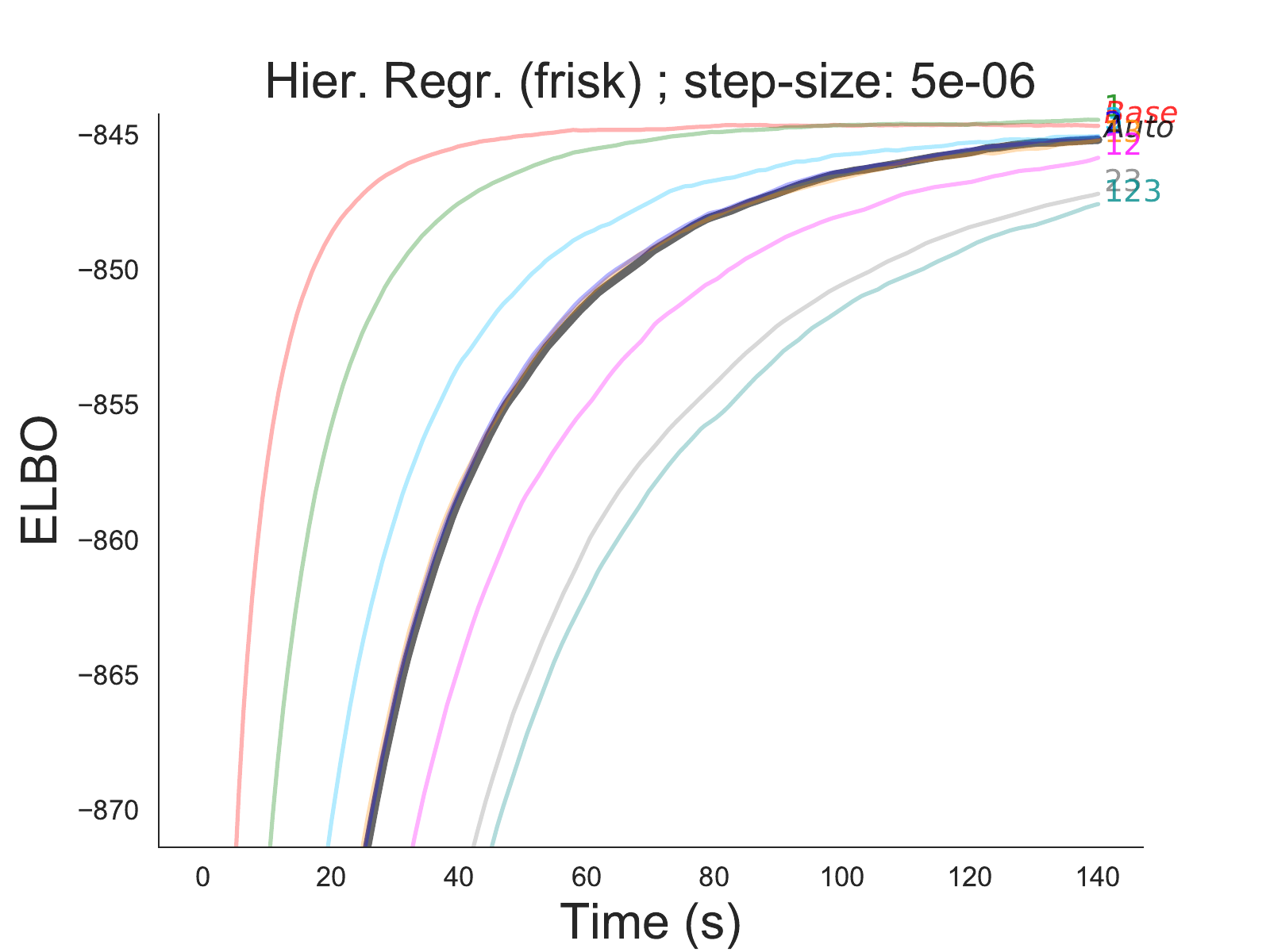}
  
  \includegraphics[width=0.45\linewidth]{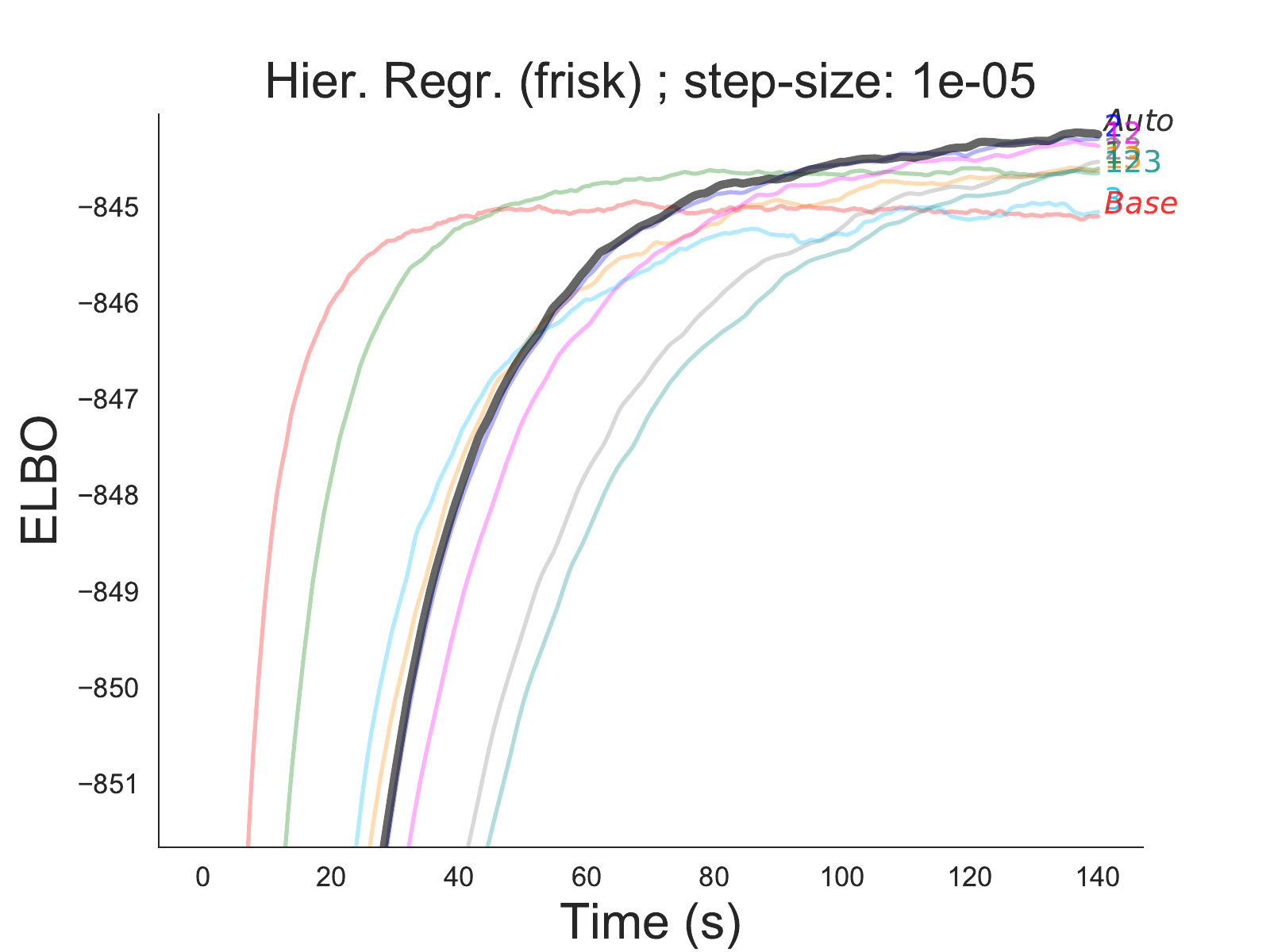}
  \includegraphics[width=0.45\linewidth]{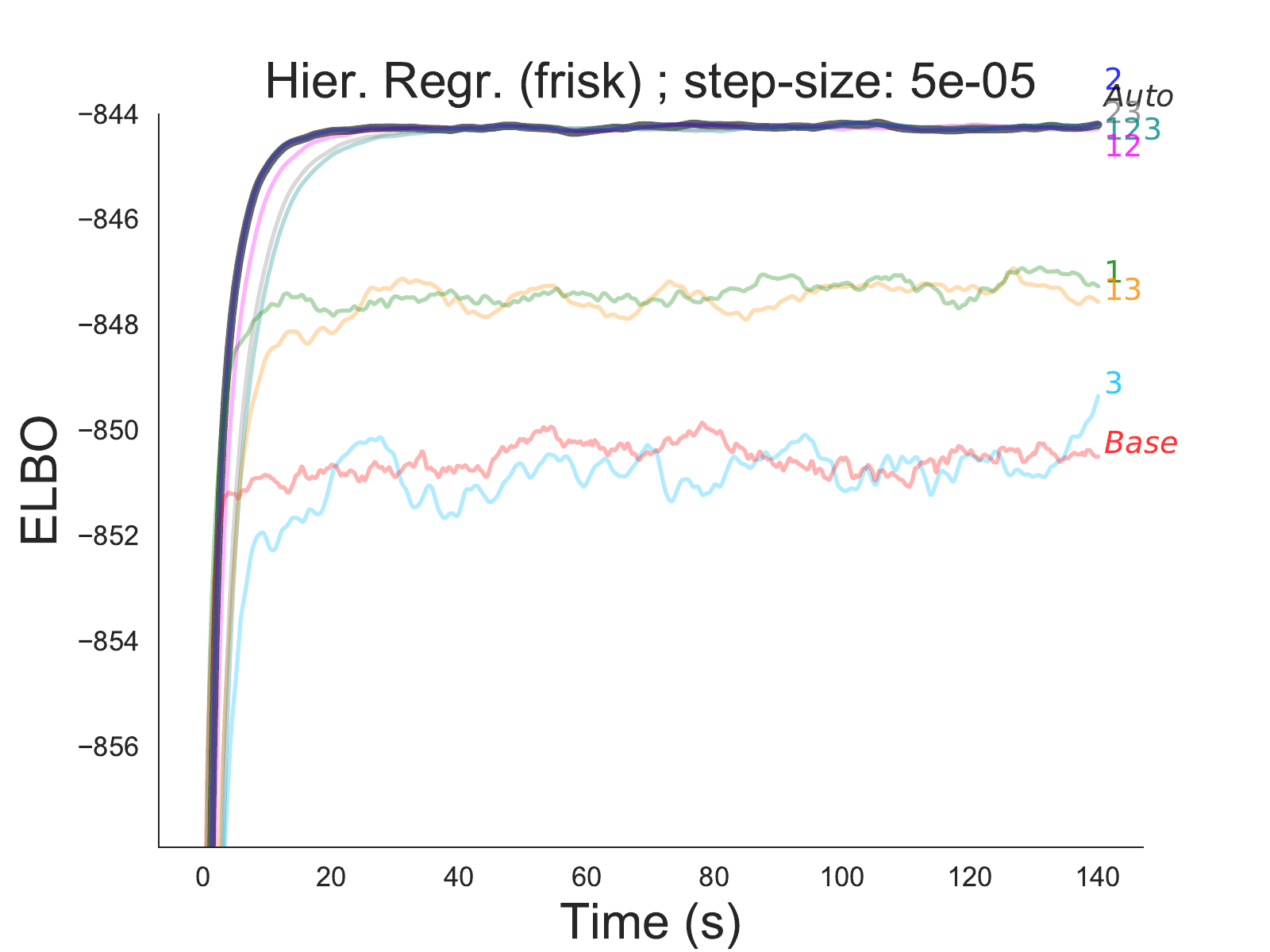}

\end{figure*}


\begin{figure*}[ht]
  \centering
  \includegraphics[width=0.45\linewidth]{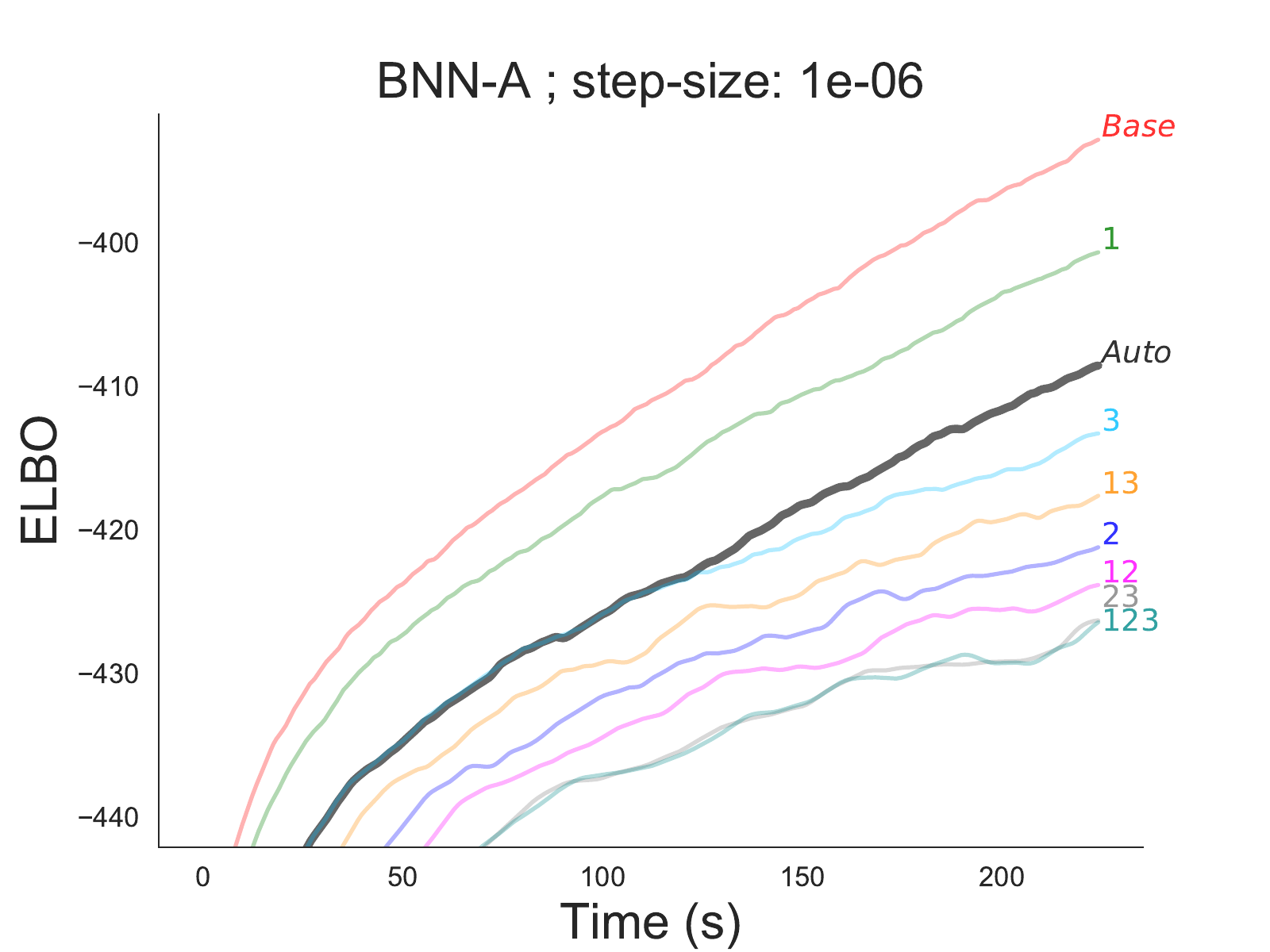}
  \includegraphics[width=0.45\linewidth]{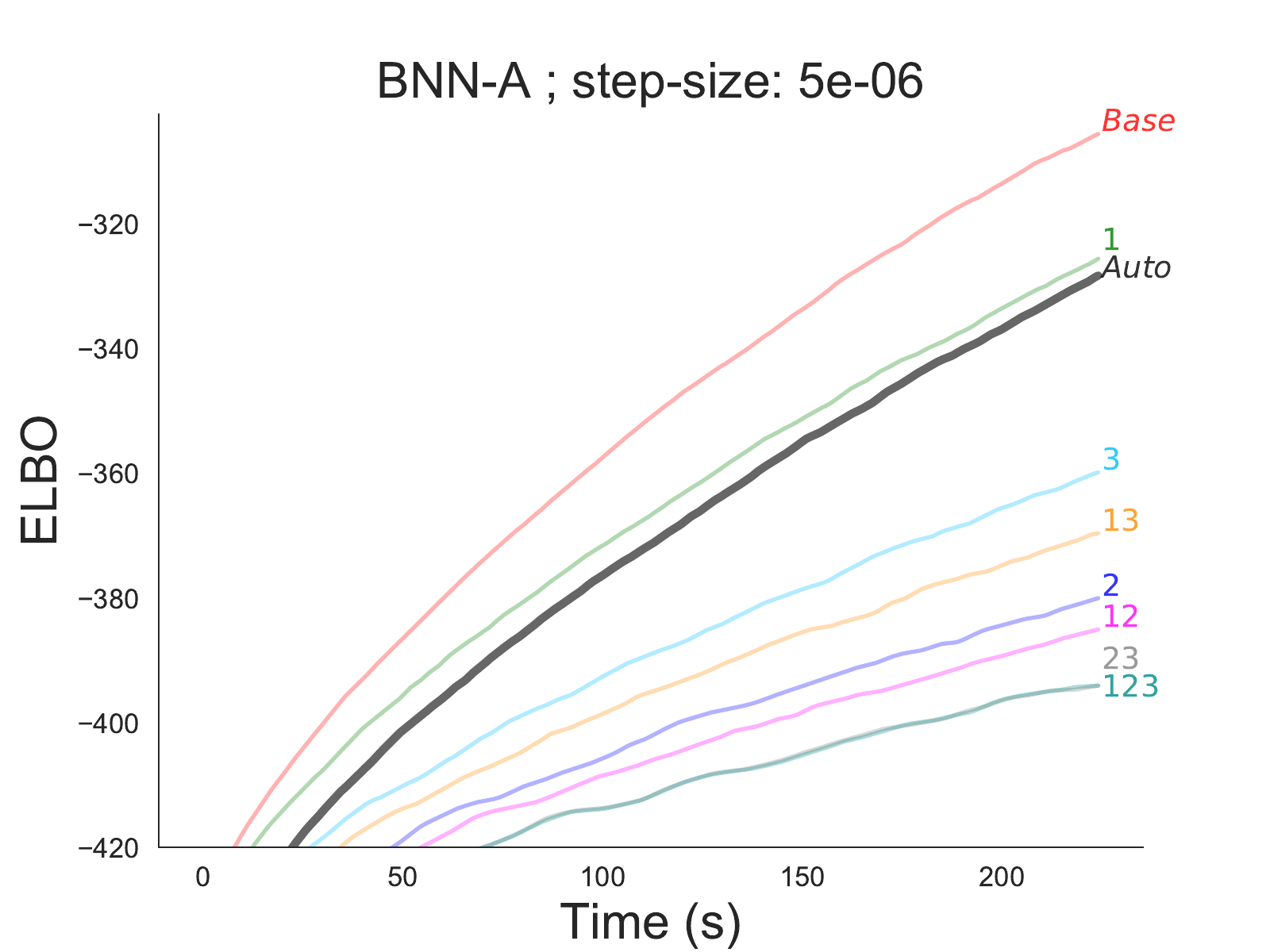}
  
  \includegraphics[width=0.45\linewidth]{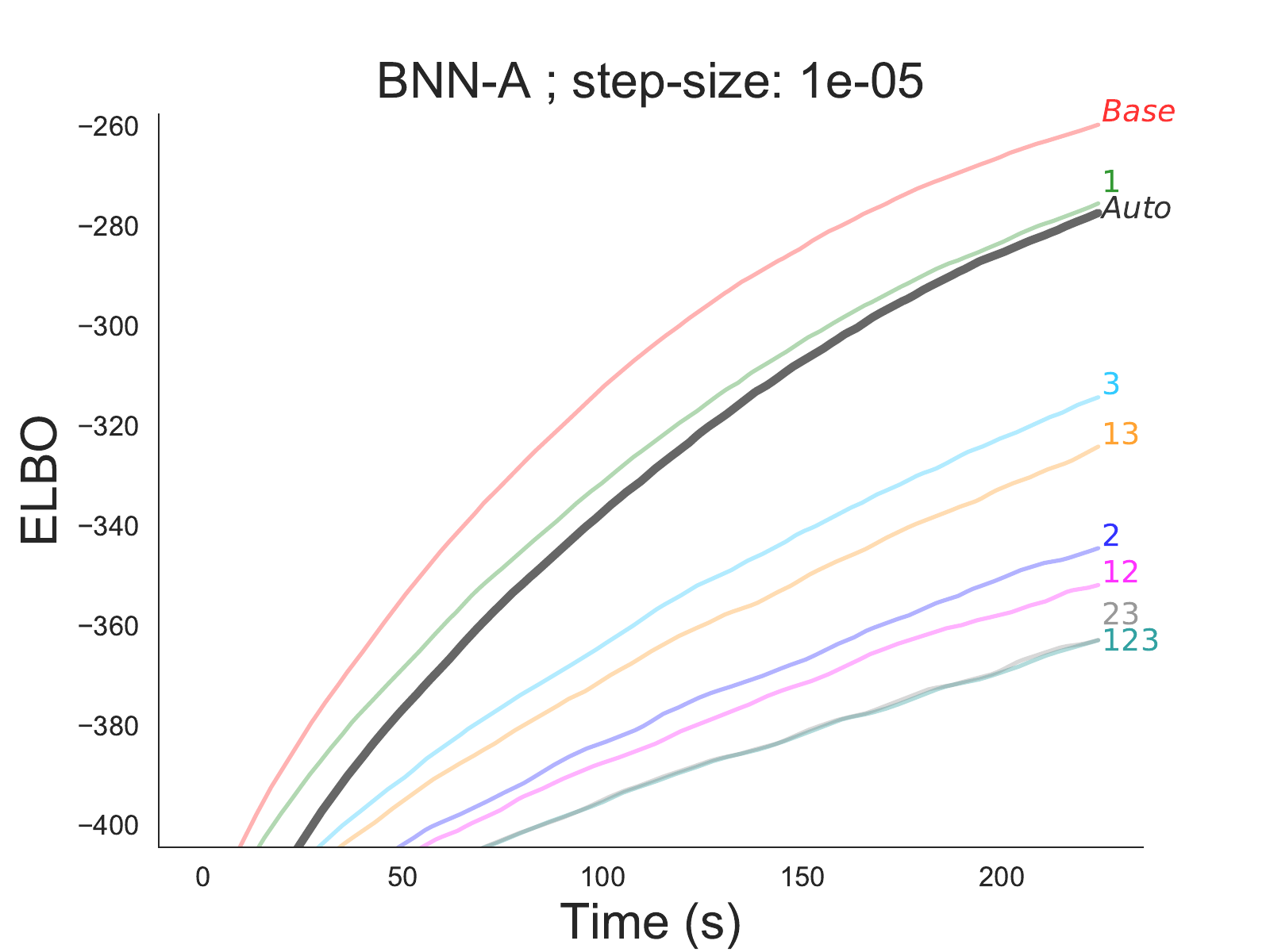}
  \includegraphics[width=0.45\linewidth]{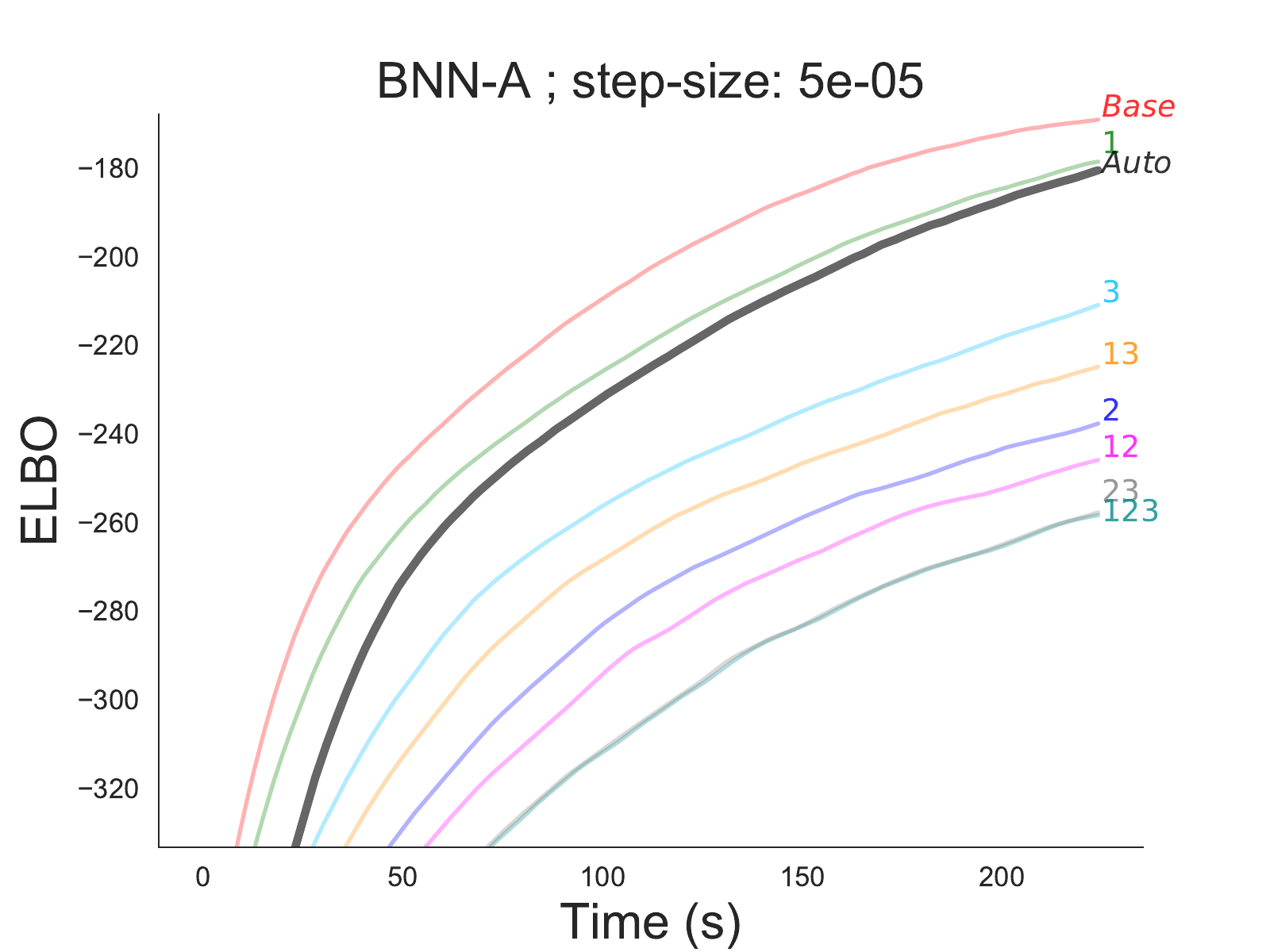}
  
  \includegraphics[width=0.45\linewidth]{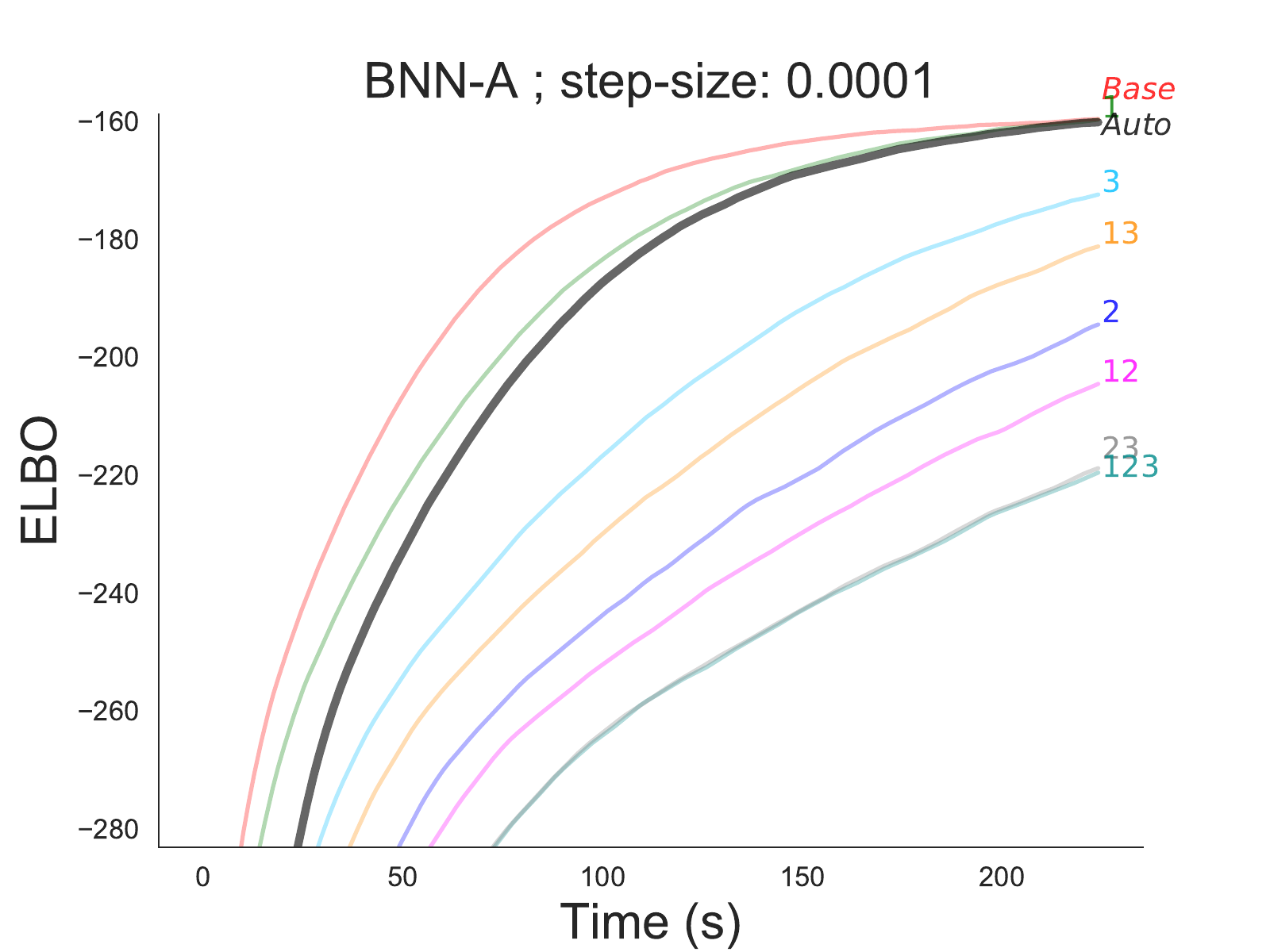}
  \includegraphics[width=0.45\linewidth]{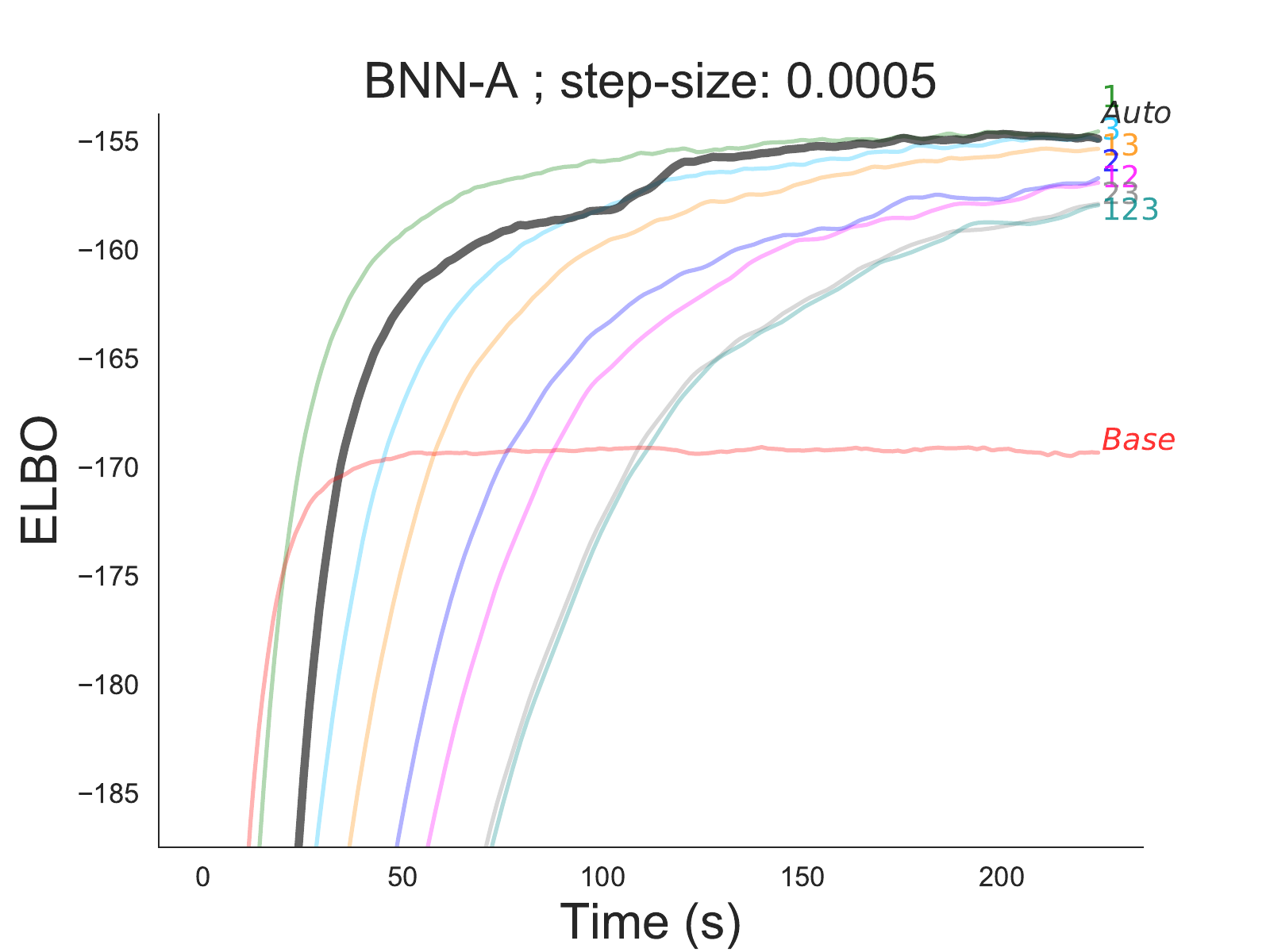}
  
  \includegraphics[width=0.45\linewidth]{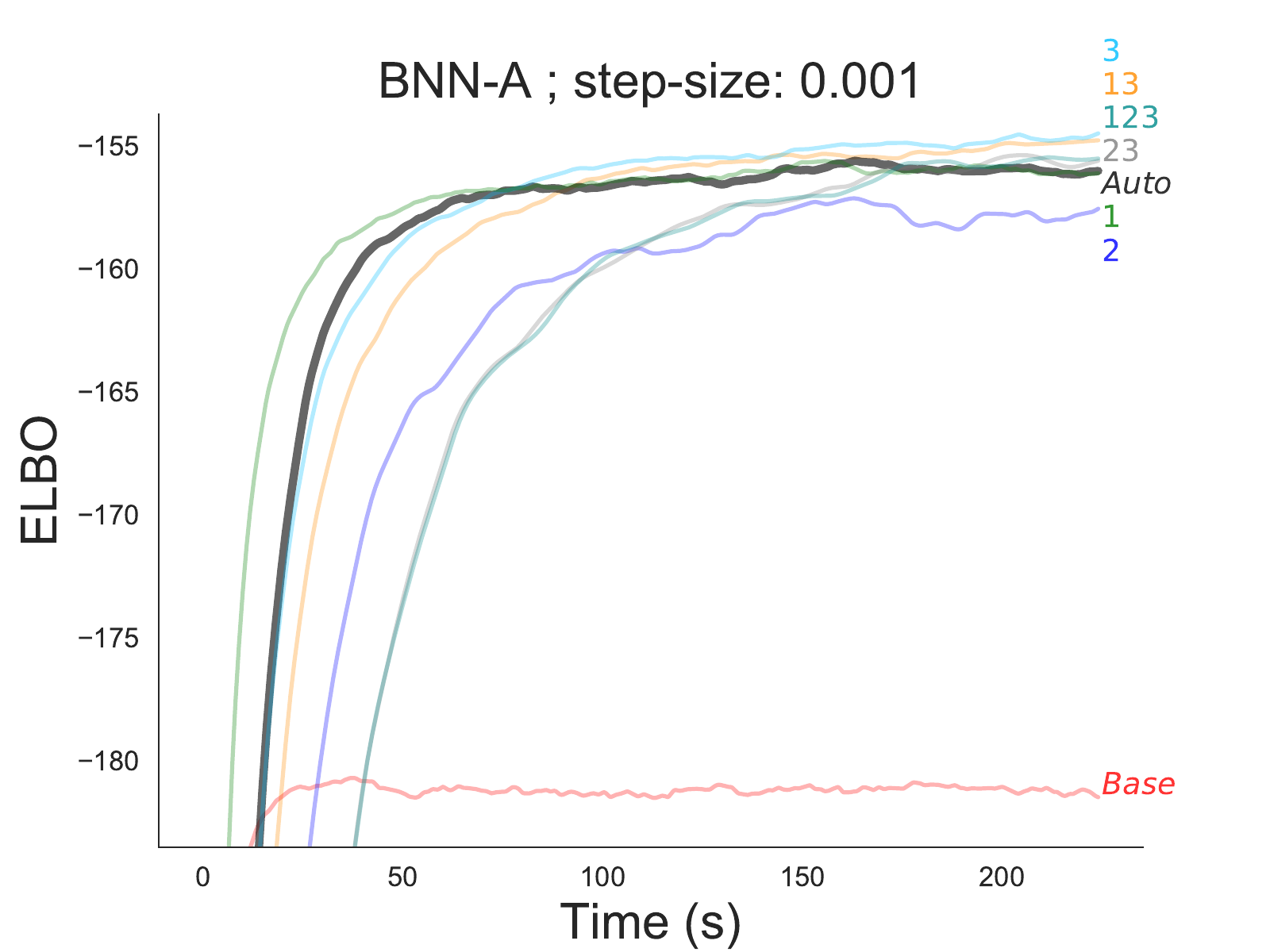}
\end{figure*}


\begin{figure*}[ht]
  \centering
  \includegraphics[width=0.45\linewidth]{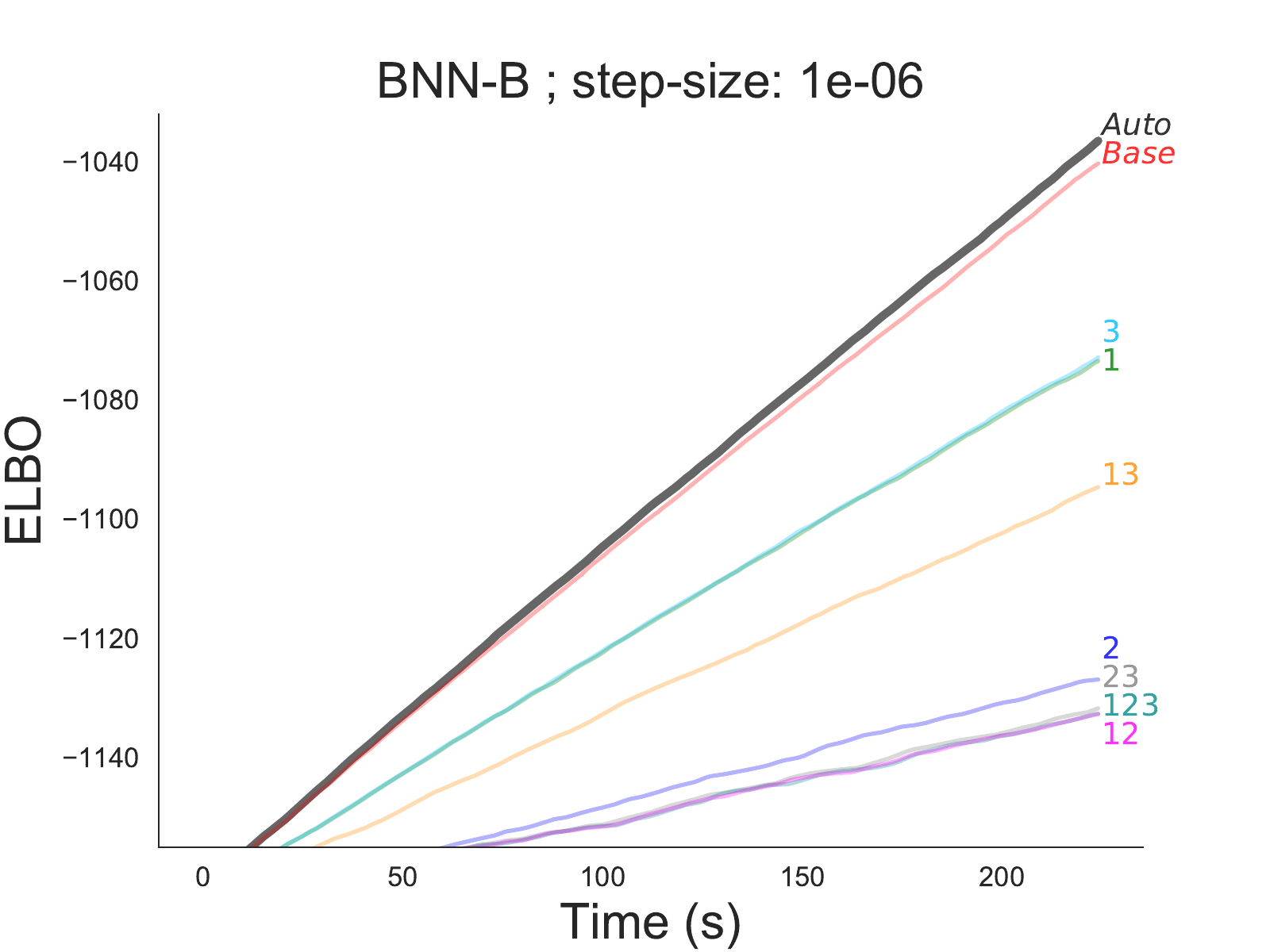}
  \includegraphics[width=0.45\linewidth]{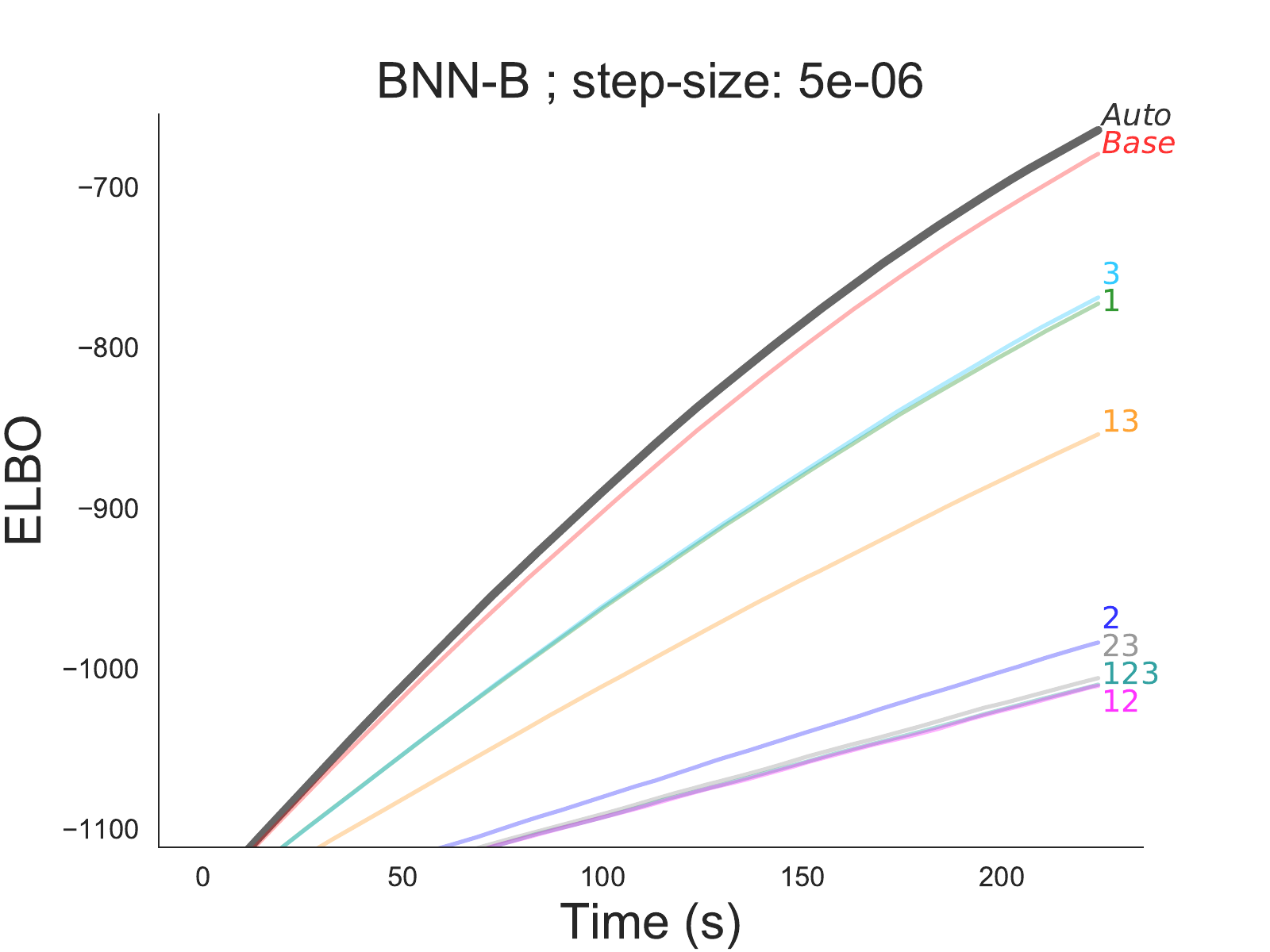}
  
  \includegraphics[width=0.45\linewidth]{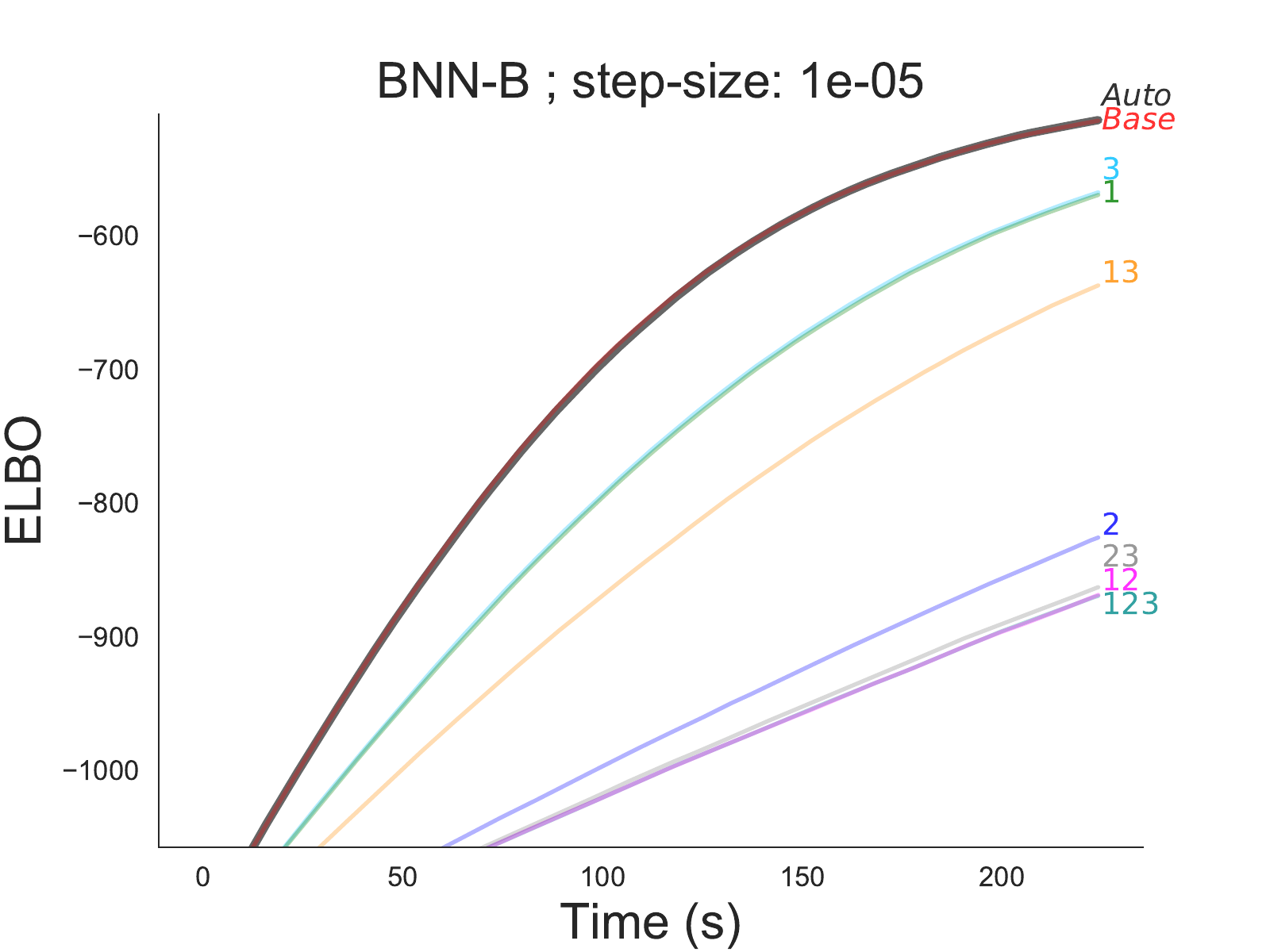}
  \includegraphics[width=0.45\linewidth]{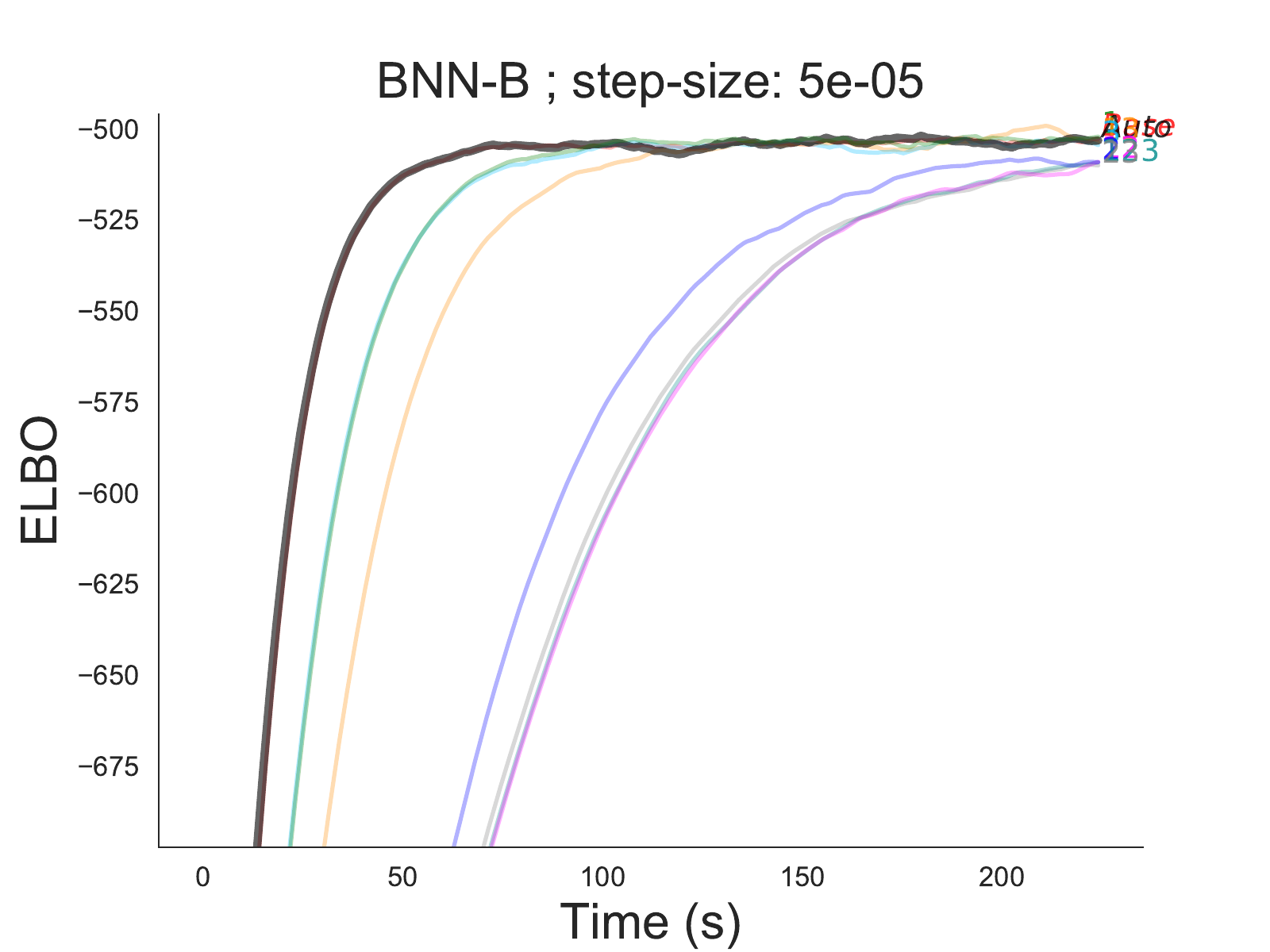}
  
  \includegraphics[width=0.45\linewidth]{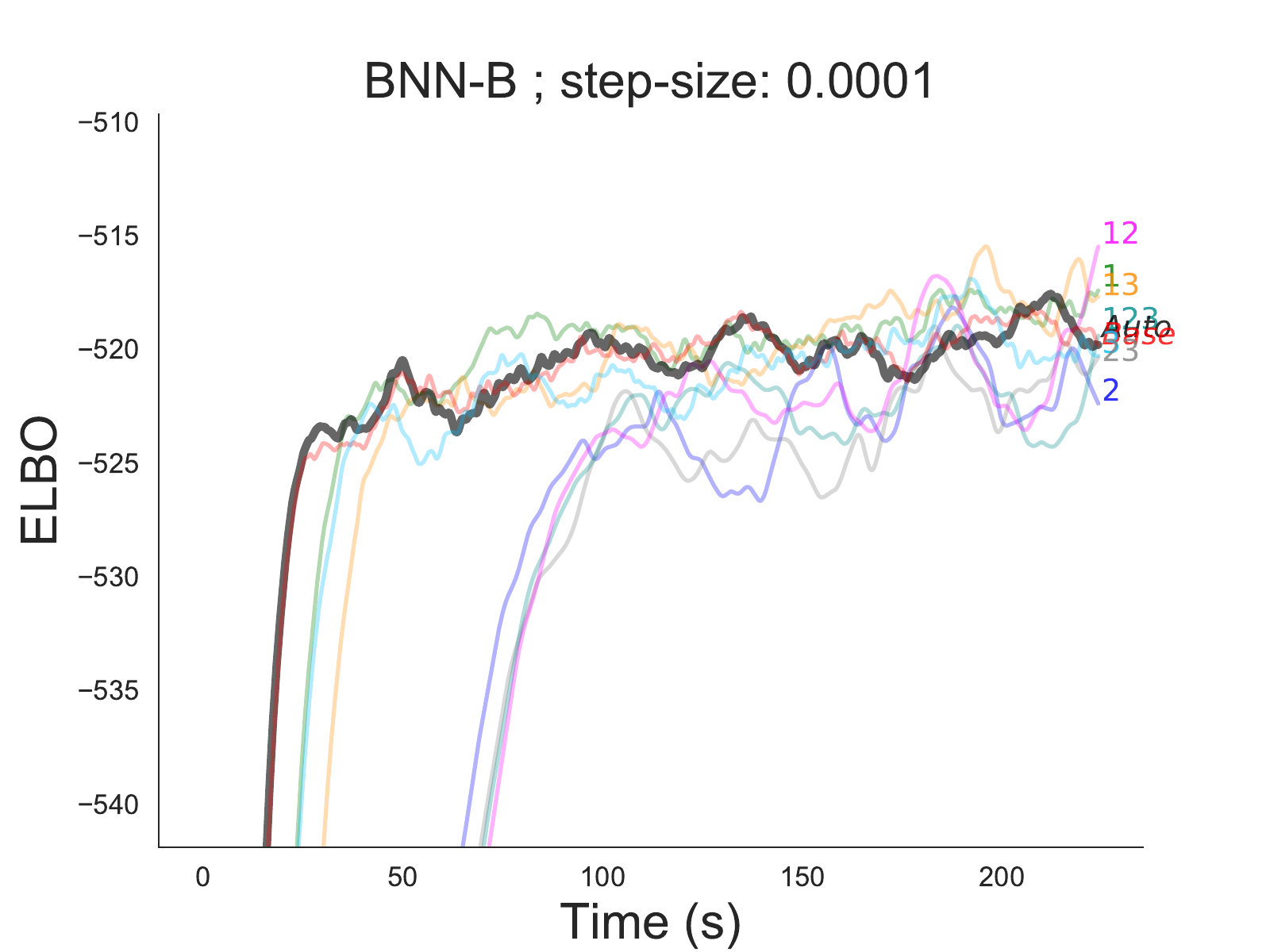}
  \includegraphics[width=0.45\linewidth]{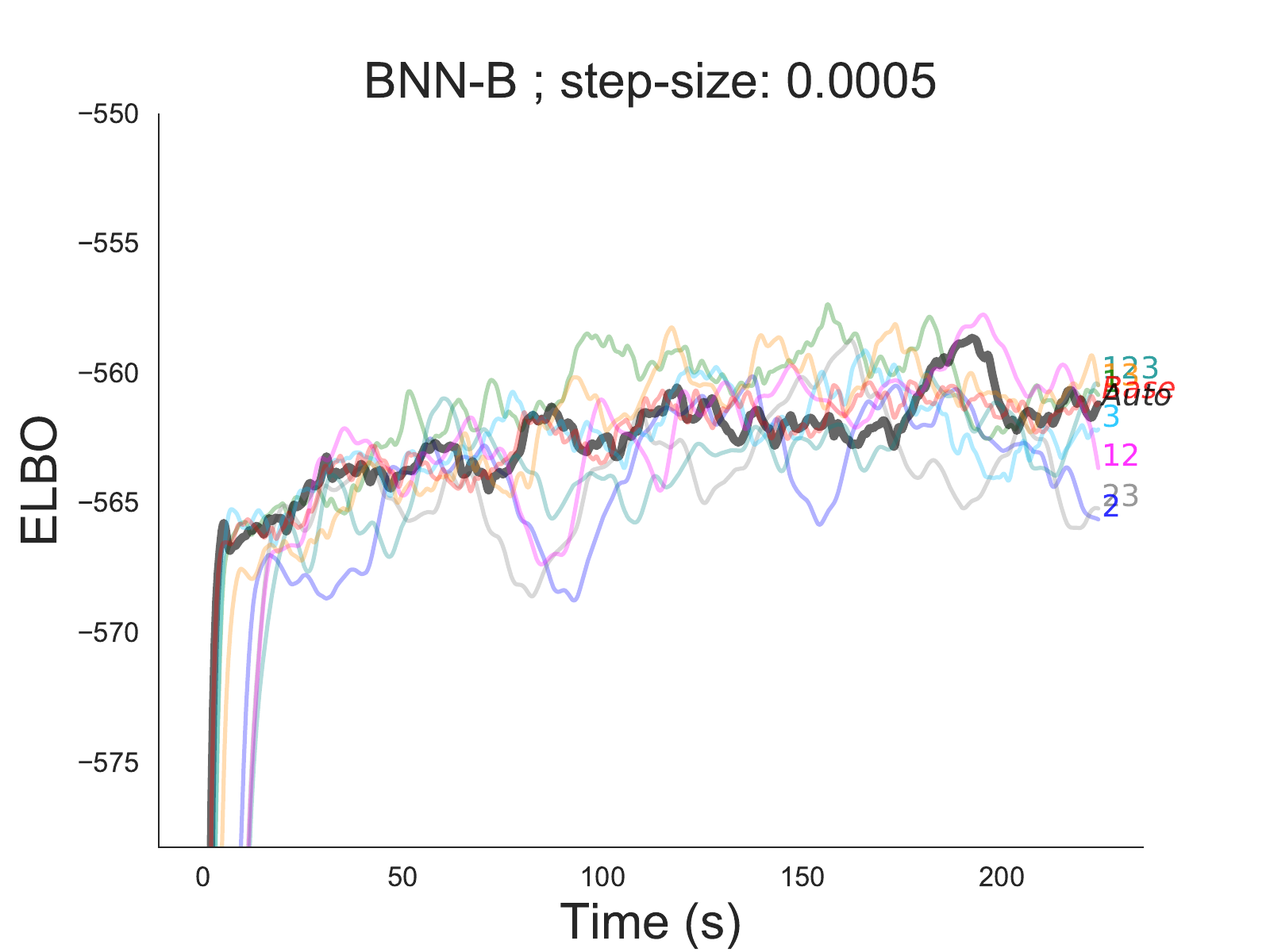}
  
  \includegraphics[width=0.45\linewidth]{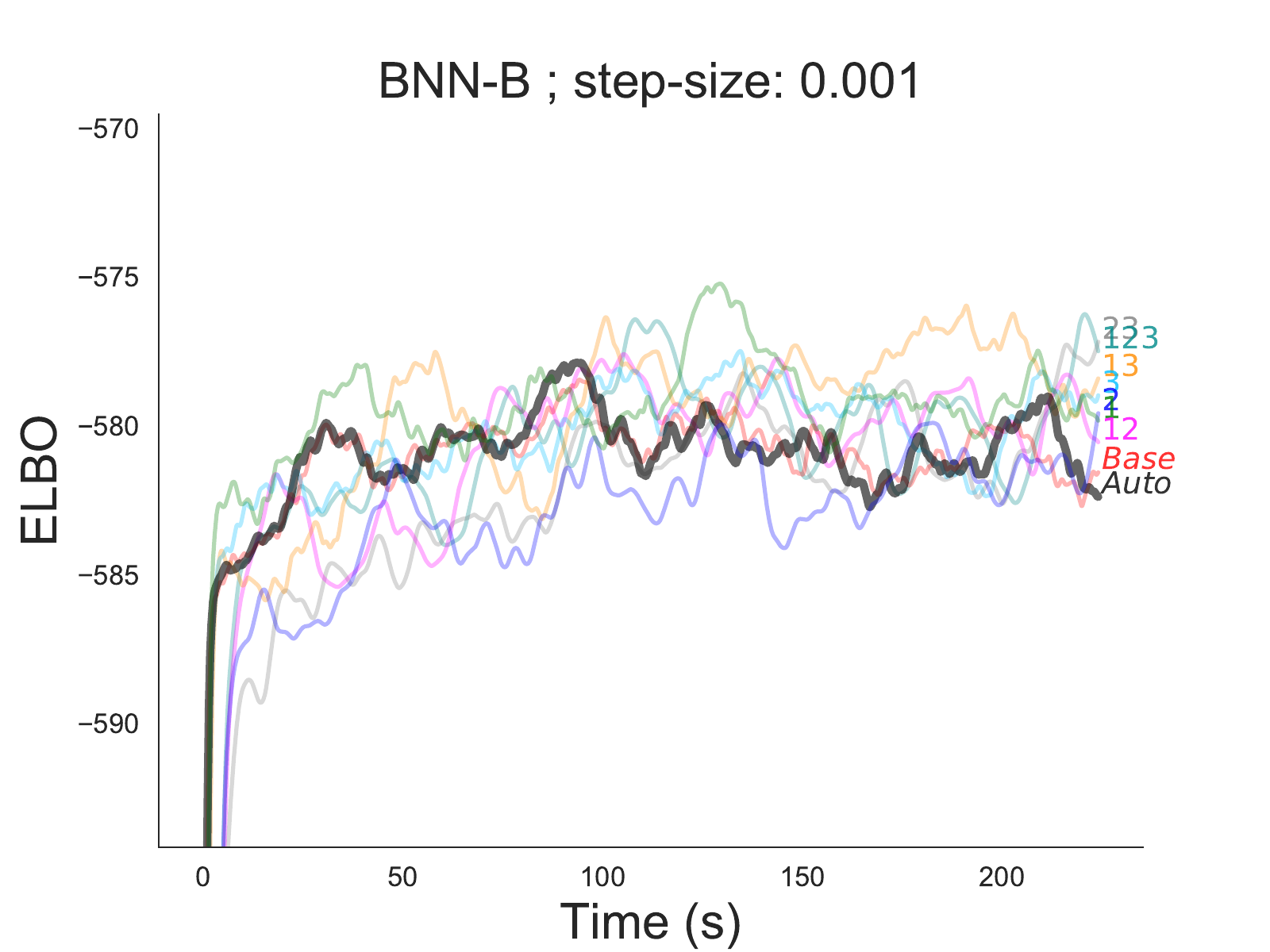}
\end{figure*}